\documentclass[onecolumn,10pt,journal]{IEEEtran}
\usepackage{cite}
\usepackage{amsmath,amssymb,amsfonts}
\usepackage{amsthm}
\usepackage{mathtools}
\usepackage{algorithmic}
\usepackage{graphicx}
\usepackage{booktabs}
\usepackage[caption=false,font=footnotesize]{subfig}
 \usepackage{multirow}
\usepackage{textcomp}
\usepackage{xcolor}
 \usepackage{url} 
\usepackage{wrapfig}
\usepackage{enumitem}
\usepackage{array} 
\usepackage{capt-of}   
\usepackage{tikz}
\usetikzlibrary{positioning,arrows.meta,calc}
\newtheorem{assumption}{Assumption}
\newtheorem{theorem}{Theorem}
\newtheorem{proposition}{Proposition}
\newcommand{\pa}{\mathrm{pa}}

\newtheorem{lemma}{Lemma}
\newtheorem{remark}{Remark}

\DeclareMathOperator{\Var}{Var}
\newcommand{\diag}{\mathrm{diag}}

\def\BibTeX{{\rm B\kern-.05em{\sc i\kern-.025em b}\kern-.08em
    T\kern-.1667em\lower.7ex\hbox{E}\kern-.125emX}}
\title{An Expectation-Maximization Algorithm for Domain Adaptation in Gaussian Causal Models}

\author{
\IEEEauthorblockN{Mohammad Ali Javidian}\\
\IEEEauthorblockA{\textit{Computer Science Department} \\
\textit{Appalachian State University}\\
Boone, USA \\
javidianma@appstate.edu}
\thanks{An earlier version of this work was accepted for the \textit{Proceedings of the 2025 IEEE International Conference on Data Mining (ICDM)}.}%
}

\begin{document}
\maketitle

\begin{abstract}
We study the problem of imputing a designated target variable that is systematically missing in a shifted deployment domain, when a Gaussian causal DAG is available from a fully observed source domain. We propose a unified EM-based framework that combines source and target data through the DAG structure to transfer information from observed variables to the missing target.
On the methodological side, we formulate a population EM operator in the DAG parameter space and introduce a first-order (gradient) EM update that replaces the costly generalized least-squares M-step with a single projected gradient step. Under standard local strong-concavity and smoothness assumptions and a BWY-style \cite{Balakrishnan2017EM} gradient-stability (bounded missing-information) condition, we show that this first-order EM operator is locally contractive around the true target parameters, yielding geometric convergence and finite-sample guarantees on parameter error and the induced target-imputation error in Gaussian SEMs under covariate shift and local mechanism shifts.
Algorithmically, we exploit the known causal DAG to freeze source-invariant mechanisms and re-estimate only those conditional distributions directly affected by the shift, making the procedure scalable to higher-dimensional models. In experiments on a synthetic seven-node SEM, the 64-node MAGIC-IRRI genetic network, and the Sachs protein-signaling data, the proposed DAG-aware first-order EM algorithm improves target imputation accuracy over a fit-on-source Bayesian network and a Kiiveri-style EM baseline, with the largest gains under pronounced domain shift.
\end{abstract}


\section{Introduction}\label{sec:intro}

\noindent\textbf{Domain Adaptation.}
Domain adaptation studies how to transfer predictive models learned in a \emph{source} domain to a \emph{target} domain whose data distribution differs. Two canonical shifts have been discussed in the literature:
\begin{enumerate}[leftmargin=*]
  \item \textbf{Covariate shift} occurs when the marginals of the \emph{context} variables differ between source and target, while the conditional
    $P(Y \mid X)$ remains invariant \cite{SHIMODAIRA2000,sugiyama2008direct,johansson19a}.
  \item \textbf{Label shift} (sometimes called \emph{target shift}) arises when the marginal of the \emph{label} changes across domains, but $P(X \mid Y)$ is unchanged \cite{Storkey09,zhang2013domain,Lipton18}.
\end{enumerate}
For an overview of additional domain adaptation scenarios and theoretical results, we refer the reader to \cite{redko2019advances}.
In this work, we focus on \emph{covariate shift} and \emph{local mechanism shifts} in a causal model: the target domain may modify a \emph{small subset of conditional distributions} in the DAG (e.g., the mechanism generating a designated target node $T$), while the remaining mechanisms remain invariant.

\noindent\textbf{Causal Inference for Domain Adaptation.}
Causal methods can exploit the underlying cause--effect structure in the data to guard against distributional shifts \cite{nastl2024causal,sun2021recovering,li2022invariant,teshima2020few,chen2021domain,wu2024causality,subbaswamy2022unifying}. Key approaches include:
\begin{itemize}[leftmargin=*]
  \item \textbf{Transportability} formalizes differences and commonalities between populations via \emph{selection diagrams}, using do-calculus \cite{Pearl09} to decide when interventional or observational effects can be carried over \cite{BareinboimPearl11,BareinboimPearl12,BareinboimPearl14,Correa-ijcai2019}.
  \item \textbf{Invariant causal prediction} (ICP) seeks subsets of predictors whose regression residuals exhibit identical distributions across environments \cite{peters2016causal,pfister2019invariant,pfister2019stabilizing}. Identifiability in nonlinear or partially observed settings remains challenging \cite{Glymour19}.
  \item \textbf{Graph surgery} removes unstable mechanisms from the factorization to enforce cross-domain invariance \cite{subbaswamy2018counterfactual,subbaswamy2019preventing}.
  \item \textbf{Graph pruning} frames adaptation as selecting predictor subsets that yield invariant conditionals \cite{MagliacaneNIPS18,rojas2018invariant,kouw2019review,javidian2021scalable}.
\end{itemize}
However, even when a subset \(A\) can be found that guarantees zero transfer bias (e.g., via pruning), the resulting incomplete-information bias can still yield large prediction errors.
Moreover, approaches such as graph surgery may require estimating causal effects or counterfactual reasoning, and many methods face scalability limitations.
In this paper, we take a different tack: under a linear--Gaussian SEM with a known DAG, we treat imputation in the shifted target domain as a \emph{missing-data} problem and develop an EM-based estimator whose \emph{first-order} updates admit BWY-style \cite{Balakrishnan2017EM} \emph{local} contraction and finite-sample error guarantees in the \emph{DAG parameter space}.

\begin{remark}[Local vs.\ basin-of-attraction guarantees (BWY-style)]
Geometric convergence results for EM are typically \emph{local} with respect to initialization. In particular, BWY-style analyses \cite{Balakrishnan2017EM} provide a quantitative \emph{basin of attraction} around the population global optimum (or optimal set) within which the EM/first-order EM operator is contractive, yielding geometric convergence to a fixed point that is within statistical precision of the population optimum. This should not be confused with global convergence from arbitrary initialization.
\end{remark}

\noindent\textbf{A Motivating Example.}
We work with the linear-Gaussian SEM whose causal structure is depicted in Fig.~\ref{fig:motivEXDAG}.
The seven nodes consist of two \emph{context} variables \(C_1, C_2\), two intermediate features \(Z, X\), the designated \emph{target} variable \(T\), and two downstream outcomes \(P, Y\). Concretely,
\[
\begin{aligned}
Z &= 2\,C_1 + 3\,C_2 + \varepsilon_Z,
\quad
X = 3\,C_1 + \varepsilon_X,\\
T &= \beta_{C_1\to T}\,C_1 + \beta_{X\to T}\,X + \beta_{Z\to T}\,Z + \varepsilon_T,\\
P &= T + \varepsilon_P,\quad Y= 2\,T + \varepsilon_Y,
\end{aligned}
\]
with noise terms \(\varepsilon_\bullet\sim\mathcal N(0,1)\) independent.
In the \emph{source} domain we draw each context variable \(C_i\sim\mathcal N(0,1)\); in the \emph{target} domain we introduce two forms of shift:
\begin{itemize}
  \item \textbf{Covariate shift} by shifting the marginal of \(C_2\) (e.g.\ \(C_2\sim\mathcal N(\mu_{\rm tgt},\sigma_{\rm tgt}^2)\)),
  \item \textbf{Local mechanism shift at \(T\)} by changing the conditional mechanism \(P(T\mid \pa(T))\), e.g.\ via a shift in coefficients and an intercept term:
  \[
    T \;=\; \tilde\beta_{C_1\to T}\,C_1 + \tilde\beta_{X\to T}\,X + \tilde\beta_{Z\to T}\,Z + b_{\rm tgt} + \tilde\varepsilon_T,
    \qquad \tilde\varepsilon_T\sim\mathcal N(0,\tilde\Delta_T).
  \]
\end{itemize}
Although \(T\) is completely unobserved in the target domain, it has observed descendants \((P,Y)\); under the invariant DAG structure, information about \(T\) is still present in the joint distribution of the observed variables and can be exploited by EM.
We compare three methods for imputing \(T\) under these shifts:
\((a)\) a fit-on-source Bayesian network baseline,
\((b)\) a Kiiveri-style EM baseline \cite{Kiiveri1987} treating \(T\) as latent, and
\((c)\) our proposed DAG-aware first-order EM algorithm (Sect.~\ref{sec:method}).
Subsequent results appear in Table~\ref{tab:results} and Fig.~\ref{fig:scatter_results}.

\begin{table}[!ht]
  \centering
  \caption{Average error metrics under covariate shift and local mechanism shift at $T$ for the motivating example.}
  \label{tab:results}
  \resizebox{.75\columnwidth}{!}{%
  \begin{tabular}{lcccc}
    \toprule
    \textbf{Shift scenario} & \textbf{Method} & \textbf{MAE} & \textbf{RMSE} & \(\mathbf{R}^2\) \\
    \midrule
    \multirow{3}{*}{Covariate shift}
      & Baseline (Fit-on-Source) & 0.7962 & 1.0137 & 0.9981 \\
      & Kiiveri EM               & 45.4529 & 45.4554 & --2.8735 \\
      & 1st-order EM             & \textbf{0.3331}  & \textbf{0.4273}  & \textbf{0.9997} \\
    \midrule
    \multirow{3}{*}{Mechanism shift at $T$}
      & Baseline (Fit-on-Source) & 6.1528  & 6.4643  & 0.9471 \\
      & Kiiveri EM               & 71.8909 & 73.2513 & --5.7872 \\
      & 1st-order EM             & \textbf{1.0386}  & \textbf{1.1228}  & \textbf{0.9984} \\
    \bottomrule
  \end{tabular}%
  }
\end{table}

\noindent\textbf{Discussion.}
Under covariate shift (shifting \(C_2\) only), the fit-on-source baseline degrades mildly, whereas under a local mechanism shift at \(T\) it can deteriorate substantially.
A Kiiveri-style EM procedure \cite{Kiiveri1987} is a natural baseline for Gaussian missing-data problems; however, without careful numerical safeguards and model-specific regularization, EM can converge to degenerate or poor local solutions in latent-variable likelihoods, especially under pronounced shift.
In contrast, our DAG-aware first-order EM initializes from the source estimate and uses the known causal structure to combine source and target information, yielding stable improvements even when \(T\) is entirely missing in the target domain.
On the theory side, classic results such as \cite{Balakrishnan2017EM} establish \emph{local} geometric convergence and finite-sample error bounds for (gradient) EM in canonical settings (e.g.\ Gaussian mixtures and regression with missing covariates).
Our contribution is to develop an analogous analysis \emph{in the Gaussian DAG parameterization}, showing that under standard local strong-concavity/smoothness and a BWY-style \emph{gradient stability} (bounded missing-information) condition, the resulting first-order EM operator is locally contractive and converges geometrically up to a statistical precision neighborhood.

\begin{figure}[ht]
  \centering
  \subfloat[Baseline (BN)\label{fig:bn-cov}]{
    \includegraphics[width=0.32\linewidth]{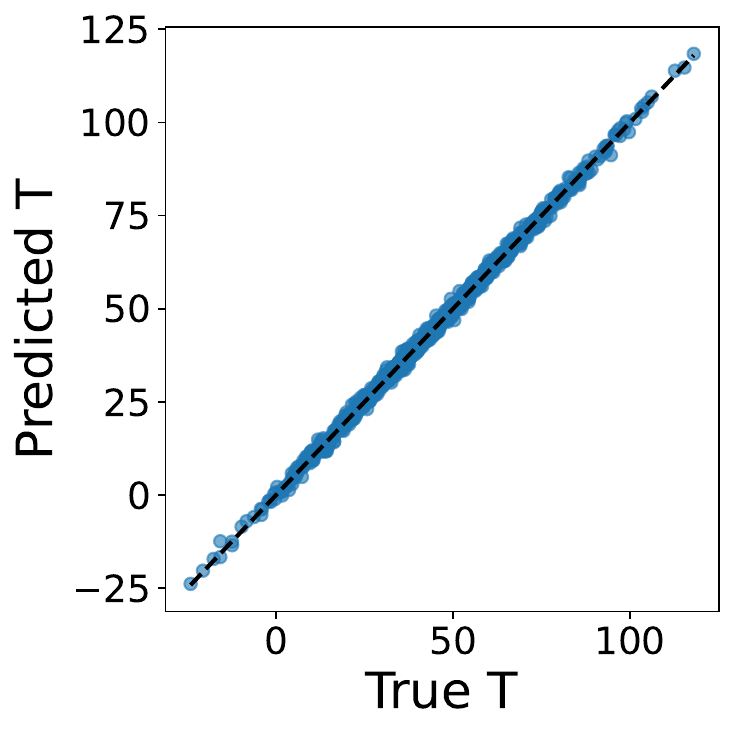}
  }%
  \subfloat[Kiiveri EM\label{fig:kiiv-cov}]{
    \includegraphics[width=0.32\linewidth]{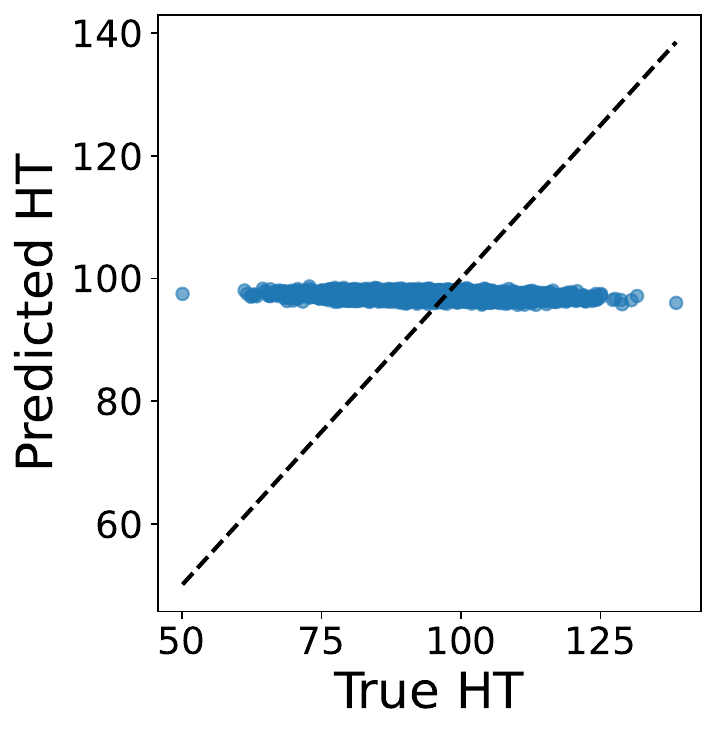}
  }%
  \subfloat[1st-order EM\label{fig:fo-cov}]{
    \includegraphics[width=0.32\linewidth]{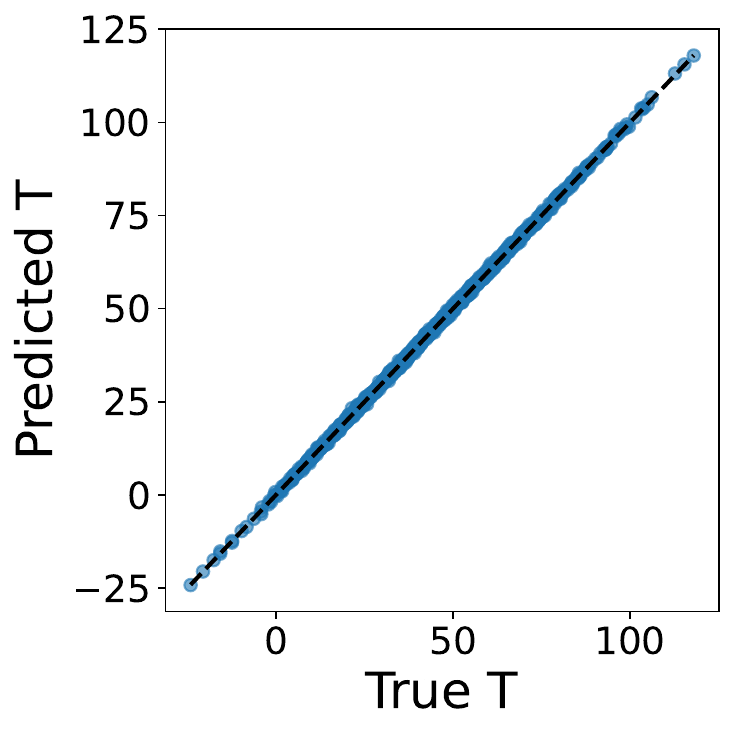}
  }

  \vspace{1em}

  \subfloat[Baseline (BN)\label{fig:bn-mech}]{
    \includegraphics[width=0.32\linewidth]{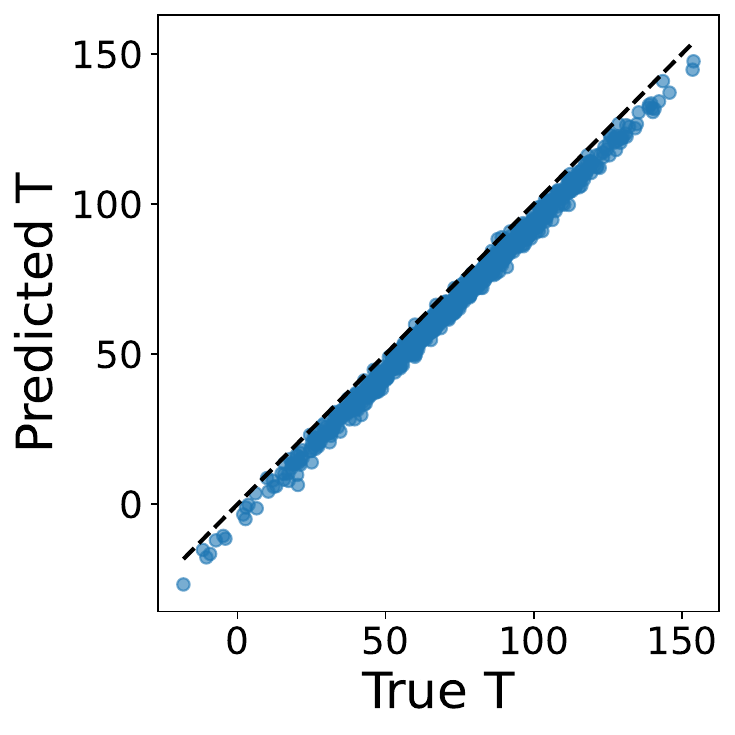}
  }%
  \subfloat[Kiiveri EM\label{fig:kiiv-mech}]{
    \includegraphics[width=0.32\linewidth]{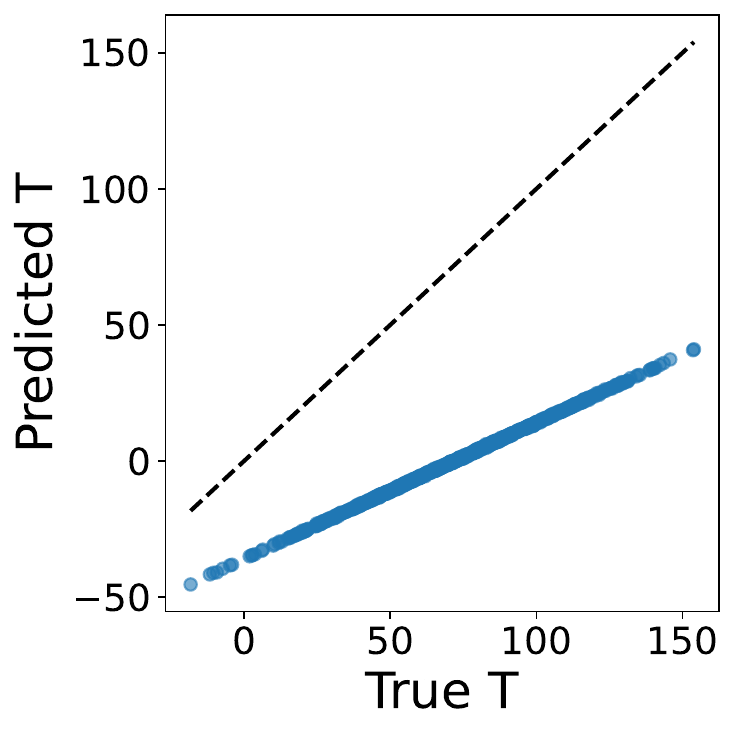}
  }%
  \subfloat[1st-order EM\label{fig:fo-mech}]{
    \includegraphics[width=0.32\linewidth]{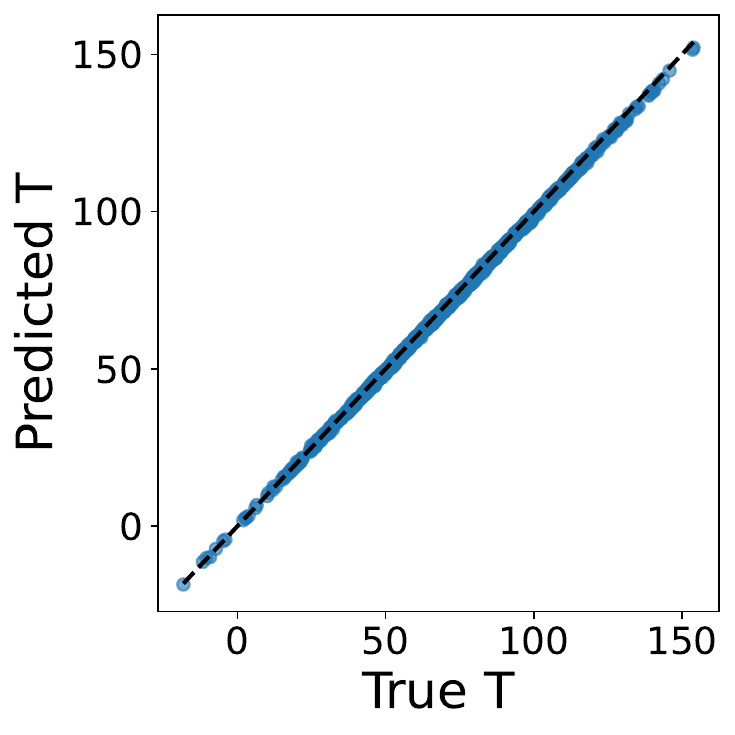}
  }

  \caption{True vs.\ predicted (imputed) \(T\) under covariate shift (top row) and a local mechanism shift at \(T\) (bottom row).
  The DAG-aware 1st-order EM achieves near-perfect recovery of \(T\) in this example, while the fit-on-source baseline degrades under mechanism shift.}
  \label{fig:scatter_results}
\end{figure}

\begin{figure}[!ht]
  \centering
  \begin{minipage}[m]{0.3\columnwidth}
    \vspace{0pt}\centering
    \resizebox{.75\linewidth}{!}{%
      \begin{tikzpicture}[
        node distance=0.5cm,
        every node/.style={draw, circle, minimum size=7mm, font=\small},
        >={Stealth[round]}, edge/.style={->, thick}
      ]
        \node (C1) {C1};
        \node[right=of C1] (C2) {C2};
        \node[below=1cm of $(C1)!0.5!(C2)$] (Z) {Z};
        \node[left= of Z]  (X) {X};
        \node[below=1cm of $(X)!0.5!(Z)$, ultra thick] (T) {T};
        \node[below left=of T]  (P) {P};
        \node[below right=of T] (Y) {Y};
        \draw[edge] (C1) -- (Z);
        \draw[edge] (C2) -- (Z);
        \draw[edge] (C1) -- (X);
        \draw[edge] (C1) -- (T);
        \draw[edge] (X)  -- (T);
        \draw[edge] (Z)  -- (T);
        \draw[edge] (T)  -- (P);
        \draw[edge] (T)  -- (Y);
      \end{tikzpicture}
    }%
  \end{minipage}\hfill
  \begin{minipage}[m]{0.65\columnwidth}
    \vspace{0pt}
    \captionof{figure}{\small Causal DAG underlying the shared data-generating process
    across source and target domains for the motivating example.
    \textbf{C$_1$}, \textbf{C$_2$}: context; \textbf{Z}, \textbf{X}: intermediates;
    \textbf{T}: target (systematically unobserved in target domain); \textbf{P}, \textbf{Y}: downstream.
    The causal structure is invariant across domains.}
    \label{fig:motivEXDAG}
  \end{minipage}
\end{figure}
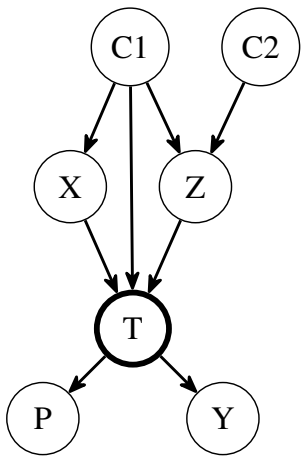

\noindent\textbf{Summary of Contributions.}
Our work provides:
\begin{itemize}[leftmargin=*]
  \item \textbf{Population EM operator in the Gaussian DAG parameterization} (Section~\ref{subsec:pop_em}):
  we characterize the population-level EM update as an operator on the DAG parameters (edge coefficients and noise variances), induced by exact conditional moments of the latent target given observed variables.
  \item \textbf{First-order (partial) M-step via gradient EM in parameter space} (Section~\ref{subsec:gradient_em}):
  we replace the $O(p^3)$ generalized least-squares (GLS) M-step with a single projected gradient step on the updatable parameter block, reducing the per-iteration cost to $O(|E|)$--$O(p^2)$ in sparse graphs (depending on the required linear solves), while maintaining ascent in the EM surrogate objective for an appropriate step size.
  \item \textbf{Domain-adaptive EM via freezing invariant mechanisms} (Section~\ref{subsec:sample_em}):
  we develop an EM routine that freezes source-invariant mechanisms and re-estimates only those conditional distributions directly affected by the shift (e.g.\ the mechanism at $T$), yielding a scalable procedure for high-dimensional DAGs.
  \item \textbf{BWY-style local geometric convergence and finite-sample error bounds} (Sections~\ref{subsec:pop_contraction}--\ref{subsec:sample_concentration}):
  building on \cite{Balakrishnan2017EM}, we prove local contraction of the population first-order EM operator under local strong-concavity/smoothness and a gradient-stability (bounded missing-information) condition, and extend the result to high-probability sample-level bounds that translate into guarantees on the induced imputation error.
\end{itemize}

\section{Related Work}

\paragraph{The Classical EM Algorithm and Early Variants.}
Dempster, Laird, and Rubin established the EM algorithm as a general-purpose method for maximum-likelihood estimation with incomplete data, proving that each iteration does not decrease the observed-data likelihood and that EM converges to a stationary point under mild conditions \cite{Dempster1977}. Louis \cite{Louis1982finding} derived the missing-information identity, decomposing the observed-data Hessian into a complete-data term plus a missing-information correction, thereby clarifying how local curvature and the fraction of missing information govern EM's convergence behavior. Wu \cite{Wu1983} analyzed EM as a generalized ascent method and proved convergence to a stationary point, with an asymptotic linear rate characterized by the local Jacobian of the EM map. Subsequent surveys and monographs (e.g., \cite{McLachlanKrishnan1997}) compile these classical guarantees and discuss practical issues such as initialization and local optima.

\paragraph{EM for Covariance Structure and Structural Equation Models.}
Gaussian structural equation models (SEMs) with latent or missing variables naturally fit the EM framework. Early covariance-structure estimation work in psychometrics and SEMs includes iterative procedures described by J\"oreskog and S\"orbom \cite{JoreskogSorbom1981} and McArdle and McDonald \cite{McArdleMcDonald1984}. Kiiveri \cite{Kiiveri1987} systematized the incomplete-data viewpoint for Gaussian recursive models by providing explicit expressions for the conditional moments $E[XX^{\top}\mid X_{\mathrm{obs}}]$ along with the corresponding score and observed information, enabling efficient EM- and Newton-type updates for fitting recursive and factor-analytic models with missing entries and latent constructs.

\paragraph{EM Variants, First-Order EM, and Tutorial Overviews.}
A broad family of EM variants improves computational efficiency or convergence speed. Meng and Rubin \cite{MengRubin1993} introduced ECM, which replaces a difficult M-step with simpler conditional maximizations; Liu and Rubin \cite{liu1994ecme} developed ECME, which accelerates convergence by maximizing the observed-data likelihood in selected blocks when convenient. Parameter expansion methods such as PX-EM \cite{MengVanDyk1997} improve curvature and can speed convergence. Tutorials such as \cite{Roche2012} survey GEM/ECM/ECME, Monte Carlo EM, and stochastic EM.
More recently, Balakrishnan \emph{et al.} \cite{Balakrishnan2017EM} formalized \emph{first-order} (gradient) EM, replacing the exact M-step by a single gradient-type update on the EM surrogate objective; this can substantially reduce per-iteration cost when the exact M-step is expensive.

\paragraph{Modern Statistical Guarantees: Population vs.\ Sample-Level Analyses.}
Balakrishnan, Wainwright, and Yu \cite{Balakrishnan2017EM} developed a unified framework in which the \emph{population} EM/gradient-EM operator is shown to be \emph{locally contractive} in a basin around a population optimum, under standard local regularity assumptions together with a BWY-style \emph{gradient stability} (bounded missing-information) condition. They also derived nonasymptotic, high-probability bounds showing that the \emph{sample} EM/gradient-EM iterates converge geometrically up to a statistical precision term controlled by uniform deviations (typically scaling as $O(\sqrt{d/n})$ in low-dimensional settings). Related developments in high-dimensional regimes include truncated or regularized EM analyses; for example, Wang, Xu, and Ravikumar \cite{WangXuRavikumar2019} studied truncated EM for high-dimensional Gaussian mixtures and established near-minimax rates under sparsity.

\paragraph{Other Notable EM-Related Advances.}
Extensions to large-scale and streaming settings include online EM \cite{CappeMoulinesRyden2005}, which uses stochastic approximation in place of full E-steps, and mini-batch stochastic EM variants \cite{CaffoJank2005}. Monte Carlo EM and stochastic EM \cite{CeleuxDiebolt1986,CaffoJank2005} approximate intractable E-steps via Monte Carlo or MCMC. Variational EM extends EM-style updates to approximate Bayesian inference by optimizing a lower bound on the marginal likelihood \cite{Bishop2006}. Collectively, these advances enable EM-like learning in settings ranging from massive datasets to complex latent-variable models.

\section{Problem Statement}\label{sec:problem}

We formalize the domain-adaptation task and specify the Gaussian DAG model and shift classes under which a BWY-style (local) geometric convergence analysis is meaningful.

\paragraph{Data and missingness.}
Let \(X=(X_1,\dots,X_p)\) be \(p\) random variables whose causal structure is a known DAG \(\mathcal{G}\) (shared across domains). One coordinate \(T=X_t\) is the \emph{designated target variable}: it is fully observed in the source domain but \emph{systematically missing} in the target domain. We write
\[
  \mathcal{D}_{\mathrm{s}}
  \;=\;
  \bigl\{\,X_{\mathrm{s}}^{(i)}\bigr\}_{i=1}^{N_{\mathrm{s}}},
  \qquad
  X_{\mathrm{s}}^{(i)}=\bigl(X_{1,\mathrm{s}}^{(i)},\dots,X_{p,\mathrm{s}}^{(i)}\bigr),
\]
for the complete source samples, and
\[
  \mathcal{D}_{\mathrm{t}}^{\mathrm{obs}}
  \;=\;
  \bigl\{\,X_{\mathrm{t},-t}^{(j)}\bigr\}_{j=1}^{N_{\mathrm{t}}},
  \qquad
  X_{\mathrm{t},-t}^{(j)}=
  \bigl(X_{1,\mathrm{t}}^{(j)},\dots,X_{t-1,\mathrm{t}}^{(j)},X_{t+1,\mathrm{t}}^{(j)},\dots,X_{p,\mathrm{t}}^{(j)}\bigr),
\]
for the observed target samples (with \(T\) missing), where \(X_{-t}\) denotes all coordinates except \(X_t\).

\paragraph{Gaussian DAG / SEM model.}
We assume \(X\) follows a linear-Gaussian SEM that is Markov with respect to \(\mathcal{G}\):
\[
  X_k \;=\; \sum_{j\in \pa(k)} \theta_{kj}\,X_j \;+\; \varepsilon_k,
  \qquad
  \varepsilon_k\sim\mathcal{N}(0,\sigma_k^2),\;\;
  \varepsilon_k \perp \varepsilon_\ell\;(k\neq \ell).
\]
Let \(B\in\mathbb{R}^{p\times p}\) denote the (strictly) lower-triangular coefficient matrix in a topological ordering, where
\(B_{kj}=\theta_{kj}\) if \(j\in\pa(k)\) and \(B_{kj}=0\) otherwise, and let \(\Delta:=(\sigma_1^2,\dots,\sigma_p^2)^\top\).
Define the structural matrix \(S := I - B\). Then the implied covariance is
\[
  \Sigma(\theta,\Delta)
  \;=\;
  \Bigl(S^\top \diag(\Delta)^{-1} S\Bigr)^{-1}.
\]
We treat \(\mathcal{G}\) as known (e.g., learned and validated using causal discovery and interventional refinement \cite{Rahman2021Accelerating,li2023causal}); given \(\mathcal{G}\), the SEM parameters are identifiable under standard regularity conditions for linear SEMs \cite{chen2014graphical}.

\paragraph{Shift classes and invariances.}
Source and target distributions may differ, but the DAG structure \(\mathcal{G}\) is invariant across domains. We consider:
\begin{itemize}[leftmargin=*]
  \item \textbf{Covariate shift:} the marginal distribution of context variables (and hence of \(X_{-t}\)) may change between domains, while the conditional mechanisms \(P(X_k\mid X_{\pa(k)})\) remain invariant.
  \item \textbf{Local mechanism shift at \(T\):} the target domain may modify only the conditional mechanism generating \(T\), i.e.,
  \(P_{\mathrm{tgt}}(T\mid X_{\pa(t)})\neq P_{\mathrm{s}}(T\mid X_{\pa(t)})\),
  while all other conditionals remain invariant.\footnote{%
  In a linear-Gaussian SEM, changing only the noise variance of \(T\) does not affect the conditional mean \(\mathbb{E}[T\mid X_{\pa(t)}]\). Therefore, improvements in mean-based imputation under ``target shift'' require a mechanism shift in \(P(T\mid X_{\pa(t)})\) (e.g., coefficient/intercept changes), which is the setting we consider.}
\end{itemize}
The availability of observed descendants of \(T\) in \(X_{-t}\) is what makes adaptation possible when \(T\) is systematically missing: changes in the mechanism at \(T\) can still be detected through their effect on the joint distribution of observed variables.

\medskip
\noindent\textbf{Domain Adaptation Task.}
Given \(\mathcal{D}_{\mathrm{s}}\), \(\mathcal{D}_{\mathrm{t}}^{\mathrm{obs}}\), and \(\mathcal{G}\), our goal is to impute the missing target values \(\{X_{t,\mathrm{t}}^{(j)}\}_{j=1}^{N_{\mathrm{t}}}\). Specifically, we compute
\[
  \widehat{X}_{t,\mathrm{t}}^{(j)}
  \;=\;
  \mathbb{E}_{\hat\theta_{\mathrm{tgt}}}\!\left[T \,\middle|\, X_{-t}=X_{\mathrm{t},-t}^{(j)}\right],
  \qquad j=1,\dots,N_{\mathrm{t}},
\]
where \(\hat\theta_{\mathrm{tgt}}\) denotes the (shift-adapted) target-domain SEM parameters learned by combining source information with the target observed-data likelihood. In our linear-Gaussian setting, this conditional expectation is available in closed form once the target parameters are estimated.

\section{Methodology and Theoretical Results}\label{sec:method}
In this section, we first present the population EM operator in the infinite-sample (population; no sampling error) setting, then describe a first-order (gradient) M--step, and finally give the sample-level domain-adaptive EM algorithm that jointly uses source and target data under domain shift. After outlining the algorithm, we develop the theoretical guarantees: population-level contraction, curvature decomposition, and sample-level error bounds.

\subsection{Population-EM Operator under a Known Causal DAG}
\label{subsec:pop_em}

To address the challenge of shifting mechanisms, we formulate the population EM operator \emph{directly in the Gaussian DAG parameter space}. Throughout, the DAG $\mathcal{G}$ is fixed and known, and only $T=X_t$ is systematically missing in the target domain.

\paragraph{Gaussian SEM parameterization.}
We write the node-wise SEM coefficients as $\theta_{kj}$:
\[
X_k \;=\; \sum_{j\in \pa(k)} \theta_{kj}\, X_j \;+\; \varepsilon_k,\qquad
\varepsilon_k \sim \mathcal{N}(0,\sigma_k^2),
\]
and collect them into a coefficient matrix $B\in\mathbb{R}^{p\times p}$ that is \emph{strictly lower-triangular}
under a topological ordering (parents precede children), with
$B_{kj}=\theta_{kj}$ if $j\in\pa(k)$ and $B_{kj}=0$ otherwise.
Let $\Delta=(\sigma_1^2,\dots,\sigma_p^2)^\top$ and write $\vartheta:=(B,\Delta)$.
Define $S:=I-B$. The implied covariance is
\[
\Sigma(\vartheta)\;=\;\bigl(S^\top \diag(\Delta)^{-1} S\bigr)^{-1}.
\]

\medskip
\noindent\textbf{Mean / intercept convention.}
For clarity, we either (i) assume variables are centered within each domain so that the SEM has zero intercepts and $m(\vartheta)=0$, or (ii) include an intercept by augmenting $X_{\pa(k)}$ with a constant $1$ (and then $m(\vartheta)$ is handled implicitly by this augmentation). When we write conditioning formulas with an explicit mean $m(\vartheta)$ below, it should be read as $m(\vartheta)=0$ under (i).

The complete-data log-likelihood factorizes by nodes. Since only $T$ is missing and we restrict adaptation to the $t$-mechanism, the only nontrivial EM update concerns the local parameters of node $t$.

\paragraph{The imputation step (E-step).}
Given a parameter iterate $\vartheta^{(r)}$, the conditional distribution of the missing target $T$ given the observed $X_{-t}$ is Gaussian $\mathcal{N}(\mu_t^{(r)}(x_{-t}), V_t^{(r)})$, where
\begin{align}
    \mu_{t}^{(r)}(x_{-t}) \;&=\; \mathbb{E}_{\vartheta^{(r)}}[T \mid X_{-t}=x_{-t}], \label{eq:muV_param} \\
    V_{t}^{(r)} \;&=\; \mathrm{Var}_{\vartheta^{(r)}}(T \mid X_{-t}). \notag
\end{align}
In a multivariate Gaussian, $V_t^{(r)}$ depends on $\vartheta^{(r)}$ but not on the realized value $x_{-t}$.

\textbf{Remark (conditioning on all observed variables).}
Although the structural equation for $T$ uses only $X_{\pa(t)}$, the imputation step conditions on the full observed vector $X_{-t}$:
$\mu_t^{(r)}(x_{-t})=\mathbb{E}_{\vartheta^{(r)}}[T\mid X_{-t}=x_{-t}]$.
This is beneficial when $T$ has observed descendants and/or when the joint distribution shifts across domains, because variables in
$X_{-t}\setminus X_{\pa(t)}$ can carry additional information about $T$ through the DAG-implied Gaussian dependence structure.

\paragraph{The parameter update (M-step).}
Let $\phi_t := (b_t,\sigma_t^2)$ denote the local mechanism parameters for $T$ in the \emph{natural} parameterization,
where $b_t \in \mathbb{R}^{|\pa(t)|}$ (or $\mathbb{R}^{|\pa(t)|+1}$ if an intercept is included by augmenting $X_{\pa(t)}$ with a constant $1$).
Let $\mathbb{E}_t[\cdot]$ denote expectation with respect to the \emph{target-domain marginal} of $X_{-t}$.
The population M-step for node $T$ reduces to least squares based on imputed moments:
\begin{equation}
\label{eq:pop_update_T}
\begin{split}
    b_t^{(r+1)} \;&=\; \Bigl(\mathbb{E}_t[\,X_{\pa(t)}X_{\pa(t)}^\top\,]\Bigr)^{-1}
    \mathbb{E}_t\!\left[\,X_{\pa(t)}\,\mu_t^{(r)}(X_{-t})\,\right], \\[1ex]
    (\sigma_t^2)^{(r+1)} \;&=\; \mathbb{E}_t\!\left[\, V_t^{(r)} \;+\; \Bigl(\mu_t^{(r)}(X_{-t}) - b_t^{(r+1)\top}X_{\pa(t)}\Bigr)^2 \right].
\end{split}
\end{equation}
We assume $\mathbb{E}_t[X_{\pa(t)}X_{\pa(t)}^\top]$ is positive definite (or the intercept-augmented analogue), ensuring the update is well-defined.

\paragraph{Population operator (restricted to updatable mechanisms).}
Collecting the node-wise maximizers yields the population EM mapping
$\vartheta^{(r+1)} = F(\vartheta^{(r)})$.
In our domain-adaptation setting, only a subset of mechanisms is updated.
In particular, when updating only the target-node mechanism, we define the restricted population operator
\[
\phi_t^{(r+1)} = F_t(\phi_t^{(r)};\,\vartheta_{\setminus t}),
\]
where $\phi_t:=(b_t,\sigma_t^2)$ denotes the local parameters of node $t$ and
$\vartheta_{\setminus t}$ denotes all frozen (source-invariant) SEM parameters held fixed during the update.

\medskip
\noindent\textbf{Log-variance reparameterization for theory.}
For the contraction and curvature analysis below, it is convenient to reparameterize
$\sigma_t^2$ by $\alpha_t:=\log\sigma_t^2$ and work with $\theta_t:=(b_t,\alpha_t)$.
This is a smooth one-to-one change of variables ($\sigma_t^2=e^{\alpha_t}$), so it induces an equivalent operator
\[
\theta_t^{(r+1)} = \widetilde F_t(\theta_t^{(r)};\,\vartheta_{\setminus t}),
\]
obtained by expressing the same update in the $(b_t,\alpha_t)$ coordinates.
We state the closed-form update in \eqref{eq:pop_update_T} using $(b_t,\sigma_t^2)$, while theoretical statements use $(b_t,\alpha_t)$
where curvature in the variance coordinate is better behaved.

\subsection{First-Order (Partial) M--Step via Gradient-EM}\label{subsec:gradient_em}

This subsection develops our \emph{first-order} M-step for Gaussian SEMs with a known DAG. Rather than maximizing the EM surrogate exactly, we take a single projected gradient-ascent step on the active mechanism parameters, yielding a valid \emph{generalized EM} (GEM) procedure: with a suitable step size, each iteration provably increases the EM surrogate $\widehat Q(\cdot\mid \vartheta^{(r)})$.

\paragraph{Finite-sample E-step and EM surrogate.}
Let $n:=N_{\mathrm{t}}$ and let $x_{-t}^{(i)}$ denote the observed coordinates in the target domain. Given a current iterate $\vartheta^{(r)}=(B^{(r)},\Delta^{(r)})$, the E-step computes conditional moments of the missing $T\mid X_{-t}$ under $\vartheta^{(r)}$, and forms the empirical EM surrogate
\begin{equation}\label{eq:Qhat_param}
\widehat{Q}\bigl(\vartheta \mid \vartheta^{(r)}\bigr)
\;:=\;
-\frac{1}{n}\sum_{i=1}^n
\mathbb{E}_{\vartheta^{(r)}}\!\left[\,
\ell_{\mathrm{comp}}\!\left(X^{(i)};\vartheta\right)
\,\middle|\, X^{(i)}_{-t}=x^{(i)}_{-t}
\right],
\end{equation}
where $\ell_{\mathrm{comp}}$ is the complete-data negative log-likelihood.
We \emph{maximize} $\widehat Q(\cdot\mid\vartheta^{(r)})$ (equivalently, minimize the expected complete-data NLL).
Since $\ell_{\mathrm{comp}}$ is the complete-data \emph{negative} log-likelihood, the quantity $\widehat Q(\vartheta\mid\vartheta^{(r)})$
is (up to an additive constant) the empirical expected complete-data \emph{log}-likelihood. Hence, for fixed $(\sigma_t^2)^{(r)}$,
$\widehat Q(\cdot\mid\vartheta^{(r)})$ is a concave quadratic function of $b_t$.

\paragraph{Active block and imputed sufficient statistics.}
We freeze source-invariant mechanisms and update only the shifted mechanism(s). For clarity, we present the update for the conditional at the missing target node $T$.
Define the empirical second moment of the observed parents
\[
\widehat{M}_{\pa(t)}
\;:=\;
\frac{1}{n}\sum_{i=1}^n
x_{\pa(t)}^{(i)}x_{\pa(t)}^{(i)\top},
\]
which is iteration-invariant since $X_{\pa(t)}\subseteq X_{-t}$ is observed.
Define the imputed cross-moment
\begin{equation}\label{eq:imputed_stats}
\widehat{v}^{(r)}_{t}
\;:=\;
\frac{1}{n}\sum_{i=1}^n
\mathbb{E}_{\vartheta^{(r)}}\!\left[
X_{\pa(t)}^{(i)}\,T^{(i)}\,\middle|\,X_{-t}^{(i)}=x_{-t}^{(i)}
\right]
\;=\;
\frac{1}{n}\sum_{i=1}^n
x_{\pa(t)}^{(i)}\,\mu_t^{(r)}(x_{-t}^{(i)}).
\end{equation}

\paragraph{Gradient for the target coefficients.}
Let $b_t\in\mathbb{R}^{|\pa(t)|}$ denote the coefficient vector for the parents of $T$, and let $\sigma_t^2$ denote its noise variance.
Viewing $\widehat Q(\cdot\mid\vartheta^{(r)})$ as a function of $b_t$ with $\sigma_t^2$ fixed at $(\sigma_t^2)^{(r)}$, differentiation yields
\begin{equation}\label{eq:grad_bt}
\nabla_{b_t}\widehat{Q}\bigl(\vartheta \mid \vartheta^{(r)}\bigr)
\;=\;
\frac{1}{(\sigma_t^2)^{(r)}}\Bigl(\widehat{v}^{(r)}_{t}-\widehat{M}_{\pa(t)}\,b_t\Bigr).
\end{equation}
We then perform a single gradient-ascent step evaluated at $b_t=b_t^{(r)}$:
\begin{equation}\label{eq:grad_em_update_bt}
b_t^{(r+1)}
\;=\;
b_t^{(r)} + \frac{\eta_r}{(\sigma_t^2)^{(r)}}\Bigl(\widehat{v}^{(r)}_{t}-\widehat{M}_{\pa(t)}\,b_t^{(r)}\Bigr),
\end{equation}
followed by projection onto the known sparsity pattern (trivial here since $b_t$ only indexes $\pa(t)$).

\begin{lemma}[GEM ascent for the one-step update]\label{lem:gem_ascent}
Fix $\vartheta^{(r)}$ and consider $\widehat{Q}(\cdot\mid \vartheta^{(r)})$ as a function of $b_t$ with $\sigma_t^2$ fixed at $(\sigma_t^2)^{(r)}$.
Then $\widehat{Q}(\cdot\mid \vartheta^{(r)})$ is concave and gradient-Lipschitz (i.e., $L^{(r)}$-smooth) in $b_t$, with
\(
L^{(r)} \;=\; \lambda_{\max}\!\bigl(\widehat{M}_{\pa(t)}\bigr)\,/\,(\sigma_t^2)^{(r)}.
\)
Moreover, if $0<\eta_r \le 2/L^{(r)}$, the update \eqref{eq:grad_em_update_bt} (gradient ascent on a concave $L^{(r)}$-smooth function) satisfies
\[
\widehat{Q}\!\left(b_t^{(r+1)} \mid \vartheta^{(r)}\right)
\;\ge\;
\widehat{Q}\!\left(b_t^{(r)} \mid \vartheta^{(r)}\right),
\]
and therefore constitutes a GEM step \cite{Dempster1977,Wu1983}.
\end{lemma}

\paragraph{Variance update (optional; closed form).}
If the mechanism shift at $T$ also affects $\sigma_t^2$, one may update it in closed form after updating $b_t$.
Specifically, using the same imputed moments computed under $\vartheta^{(r)}$ (an ECM/GEM-style update),
\begin{equation}\label{eq:sigma_update}
(\sigma_t^2)^{(r+1)}
=
\frac{1}{n}\sum_{i=1}^n
\left[
V_t^{(r)}
+\Bigl(\mu_t^{(r)}(x_{-t}^{(i)}) - b_t^{(r+1)\top}x_{\pa(t)}^{(i)}\Bigr)^2
\right].
\end{equation}
Equivalently, one may update $\alpha_t^{(r+1)}:=\log(\sigma_t^2)^{(r+1)}$.

\paragraph{Complexity.}
Computing \eqref{eq:grad_bt} costs $O(n\,|\pa(t)|^2)$ to aggregate $\widehat{M}_{\pa(t)}$ and $\widehat{v}^{(r)}_{t}$, plus the cost of evaluating the Gaussian conditional moments in the E-step.

\noindent\textbf{E-step conditioning cost (no explicit matrix inversion).}
We do not form $\Sigma(\vartheta)$ explicitly. Let $K(\vartheta):=\Sigma(\vartheta)^{-1}=S^\top \diag(\Delta)^{-1}S$ denote the precision implied by the SEM (typically sparse for sparse DAGs, with sparsity pattern related to the induced Gaussian Markov graph / moralized DAG).
For a single missing coordinate $T=X_t$ in a Gaussian model, conditioning can be expressed directly in terms of the precision:
\[
V_t^{(r)}=\Var_{\vartheta^{(r)}}(T\mid X_{-t}) = \big(K^{(r)}_{tt}\big)^{-1},\qquad
\mu_t^{(r)}(x_{-t}) = m_t^{(r)} - \big(K^{(r)}_{tt}\big)^{-1}K^{(r)}_{t,-t}\big(x_{-t}-m_{-t}^{(r)}\big),
\]
where $m^{(r)}:=\mathbb{E}_{\vartheta^{(r)}}[X]$ (equal to $0$ under the centering convention above).
If an intercept is included via parent augmentation, the same formulas apply with the augmented design.
Thus, per sample, evaluating $\mu_t^{(r)}(x_{-t})$ costs $O(\mathrm{nnz}(K^{(r)}_{t,-t}))$, i.e., proportional to the number of nonzeros ($\mathrm{nnz}$) in the off-diagonal portion of row $t$ of the precision matrix.
This sparsity pattern corresponds to the neighbors of $t$ in the induced Gaussian Markov graph (equivalently, the sparsity pattern of row $t$ of $K^{(r)}$).
Since our updates modify only the active mechanism at $t$, only row $t$ of $S=I-B$ changes, hence $K=S^\top \diag(\Delta)^{-1}S$ changes only on the index set $\{t\}\cup \pa(t)$ (i.e., a local submatrix).
Accordingly, the E-step can be implemented via local sparse updates/row operations rather than forming a dense $O(p^3)$ inverse.

\subsection{Domain-Adaptive EM}\label{subsec:sample_em}

We now describe the practical \emph{domain-adaptive EM} procedure that combines fully observed source data with partially observed target data. The key idea is to use the source domain to estimate (and then \emph{freeze}) mechanisms that are assumed invariant, while adapting only the mechanism(s) affected by the shift by maximizing the \emph{target observed-data log-likelihood}
\[
\ell_{\mathrm{obs,tgt}}(\vartheta)\;:=\;\frac{1}{N_{\mathrm{t}}}\sum_{i=1}^{N_{\mathrm{t}}}\log p_{\vartheta}\!\bigl(x_{-t}^{(i)}\bigr),
\]
with $T$ treated as latent and $p_\vartheta$ induced by the Gaussian SEM under the known DAG $\mathcal{G}$. In practice we carry out this maximization via EM/GEM by increasing the empirical EM surrogate $\widehat Q(\cdot\mid\vartheta^{(r)})$.

\medskip
\noindent\textbf{1) Source-domain estimation (invariant mechanisms).}
Using the complete source samples, we fit all SEM conditionals under the known DAG $\mathcal{G}$ via node-wise least squares (equivalently, Gaussian SEM MLE under $\mathcal{G}$). Let
\[
\vartheta^{(\mathrm{s})}=(B^{(\mathrm{s})},\Delta^{(\mathrm{s})})
=\operatorname{FitDAG}\!\left(\mathcal{G},\mathcal{D}_{\mathrm{s}}\right),
\]
where $\operatorname{FitDAG}$ denotes any consistent DAG-constrained Gaussian SEM fit (e.g., regressions in a topological order, with $\sigma_k^2$ estimated from residual variance).

\medskip
\noindent\textbf{Which parameters are frozen?}
We consider the following shift models (cf.\ Section~\ref{sec:problem}):
\begin{itemize}[leftmargin=*]
  \item \textbf{Covariate/root shift (marginal interventions on observed roots/contexts).}
We allow the \emph{marginal} distribution of some observed root/context variables to change between domains (equivalently, their root mechanisms change), while keeping all \emph{non-root} conditional mechanisms $P(X_k\mid X_{\pa(k)})$ invariant.
If imputation conditions only on $X_{\pa(t)}$, then these marginal changes do not affect $\mathbb{E}[T\mid X_{\pa(t)}]$ and the source fit is sufficient.
However, if imputation conditions on a larger set $X_{-t}$ that includes descendants or other correlated variables, then $\mathbb{E}[T\mid X_{-t}]$ depends on the target-domain second-order structure. In this case we \emph{optionally} refit the shifted root marginals (e.g., their means and variances) in closed form from unlabeled target samples (or include them in the active set), while keeping the remaining mechanisms frozen.

  \item \textbf{Local mechanism shift at $T$:} only the conditional $P(T\mid X_{\pa(t)})$ may change between domains, while all other mechanisms remain invariant. In this case, we freeze $\{b_k,\sigma_k^2\}_{k\neq t}$ at their source estimates and adapt the active mechanism at node $t$ using target-domain EM/GEM updates.
\end{itemize}

\medskip
\noindent\textbf{2) Target-domain GEM updates (active mechanism).}
Initialize $\vartheta^{(0)}:=\vartheta^{(\mathrm{s})}$ and iterate for $r=0,1,2,\dots$:

\begin{itemize}[leftmargin=*]
  \item \textbf{E-step (target).}
  Using the current iterate $\vartheta^{(r)}$, compute the conditional moments
  $\mu_t^{(r)}(x^{(i)}_{-t})=\mathbb{E}_{\vartheta^{(r)}}[T\mid X_{-t}=x^{(i)}_{-t}]$ and
  $V_t^{(r)}=\Var_{\vartheta^{(r)}}(T\mid X_{-t})$
  for each target sample $x^{(i)}_{-t}$ via Gaussian conditioning (note that $V_t^{(r)}$ does not depend on $x^{(i)}_{-t}$ in the Gaussian case).

  \item \textbf{M-step (target, first-order on $b_t$ and optional closed-form variance update).}
  Form the parent sufficient statistic (constant across iterations) and the cross-moment (updated each iteration):
\[
\widehat{M}_{\pa(t)}
:=
\frac{1}{N_{\mathrm{t}}}\sum_{i=1}^{N_{\mathrm{t}}}
x_{\pa(t)}^{(i)}x_{\pa(t)}^{(i)\top},
\qquad
\widehat{v}^{(r)}_{t}
:=
\frac{1}{N_{\mathrm{t}}}\sum_{i=1}^{N_{\mathrm{t}}}
x_{\pa(t)}^{(i)}\,\mu_t^{(r)}(x_{-t}^{(i)}).
\]
  Update $b_t$ while keeping all other mechanisms fixed:
    \[
    b_t^{(r+1)}=b_t^{(r)}+\eta_r\frac{1}{(\sigma_t^2)^{(r)}}\Bigl(\widehat{v}^{(r)}_{t}-\widehat{M}_{\pa(t)}\,b_t^{(r)}\Bigr),
    \qquad
    0<\eta_r\le \frac{2(\sigma_t^2)^{(r)}}{\lambda_{\max}(\widehat{M}_{\pa(t)})},
    \]
    which ensures ascent of the EM surrogate on the $b_t$-block (Lemma~\ref{lem:gem_ascent}).
    (Alternatively, an exact M-step may be used: $b_t^{(r+1)}=\widehat{M}_{\pa(t)}^{-1}\widehat{v}^{(r)}_t$ when $\widehat{M}_{\pa(t)}$ is invertible.)
    If desired, update $\sigma_t^2$ in closed form as
    \[
    (\sigma_t^2)^{(r+1)}
    =
    \frac{1}{N_{\mathrm{t}}}\sum_{i=1}^{N_{\mathrm{t}}}
    \left[
    V_t^{(r)}
    +\Bigl(\mu_t^{(r)}(x^{(i)}_{-t})-b_t^{(r+1)\top}x^{(i)}_{\pa(t)}\Bigr)^2
    \right],
    \]
    followed by truncation
    \(
    (\sigma_t^2)^{(r+1)}\leftarrow \min\{\max\{(\sigma_t^2)^{(r+1)},\Delta_{\min}\},\Delta_{\max}\}
    \)
    if bounded-variance constraints are imposed. Equivalently, set $\alpha_t^{(r+1)}:=\log(\sigma_t^2)^{(r+1)}$.

  \item \textbf{Update implied conditioning quantities.}
  Set $\vartheta^{(r+1)}$ by replacing only the $T$-mechanism in $\vartheta^{(\mathrm{s})}$ with $(b_t^{(r+1)},(\sigma_t^2)^{(r+1)})$.
  For subsequent conditioning, recompute the required precision/conditioning quantities implied by $\vartheta^{(r+1)}$ (e.g., via $K(\vartheta)=S^\top\diag(\Delta)^{-1}S$) without forming a dense covariance matrix.

  \item \textbf{Stopping criterion.}
  Stop when the active parameters stabilize, e.g.,
  $\|b_t^{(r+1)}-b_t^{(r)}\|_2\le \varepsilon_b$ and (if updated)
  $|(\sigma_t^2)^{(r+1)}-(\sigma_t^2)^{(r)}|\le \varepsilon_\sigma$,
  or when the surrogate improvement falls below a threshold.
\end{itemize}

\medskip
\noindent\textbf{3) Target imputation.}
After convergence, impute each missing target value by the conditional mean under the adapted parameters:
\[
\widehat{X}_{t,\mathrm{t}}^{(i)}=\mathbb{E}_{\hat\vartheta_{\mathrm{tgt}}}\!\left[T\mid X_{-t}=x^{(i)}_{-t}\right],
\qquad i=1,\dots,N_{\mathrm{t}}.
\]

\medskip
\noindent\textbf{Remarks.}
\begin{itemize}[leftmargin=*]
  \item \textbf{Variance-only shift (clarification).}
  Under a linear-Gaussian SEM, changing only $\sigma_t^2$ does not change $\mathbb{E}[T\mid X_{\pa(t)}]$ because $P(T\mid X_{\pa(t)})$ retains the same conditional mean.
  However, when imputation conditions on a larger set $X_{-t}$ that includes descendants or other informative variables, a variance-only change can affect $\mathbb{E}[T\mid X_{-t}]$ through posterior precision weighting.
  Our adaptation therefore primarily targets shifts in the conditional mean mechanism (coefficients/intercept), while optionally updating $\sigma_t^2$ as above.
  Such shifts are learnable from unlabeled target data when $T$ has observed descendants (or other observed variables whose distribution depends on $T$), enabling information flow from $X_{-t}$ to the latent $T$.

  \item \textbf{Local updates and scalability.}
  The M-step updates only the active mechanism parameters and requires forming $\widehat{M}_{\pa(t)}$ and $\widehat{v}^{(r)}_t$, which costs $O(N_{\mathrm{t}}|\pa(t)|^2)$, plus the E-step cost of Gaussian conditioning.

  \item \textbf{Implementation note.}
  If desired, one may occasionally perform an exact refit of the active block to improve numerical stability.
\end{itemize}

\subsection{Population-Level Contraction in a Neighborhood of the Target Parameters}
\label{subsec:pop_contraction}

This subsection provides a BWY-style \emph{local} contraction result for our population operators,
stated in the \emph{DAG/SEM parameter space} (rather than in unconstrained covariance space).
Since our domain-adaptive procedure freezes source-invariant mechanisms and adapts only the shifted
conditional at \(T\), we analyze the \emph{active mechanism block} at node \(t\).

\paragraph{Active parameterization (log-variance).}
To obtain well-behaved curvature in the variance coordinate, we parameterize the noise variance via the log-variance
\[
\alpha_t := \log \sigma_t^2 \in \mathbb{R},
\]
and define the active block as
\[
\theta_t := (b_t,\alpha_t)\in\mathbb{R}^{|\pa(t)|}\times\mathbb{R},
\]
(with an intercept absorbed into \(b_t\) by augmenting \(X_{\pa(t)}\) with a constant \(1\), if used).
Let \(\theta_t^*\) denote the true target-domain mechanism parameters at node \(T\).

\paragraph{Population EM/GEM operators on the active block.}
Let \(\bar Q_t(\theta_t \mid \theta_t')\) denote the \emph{population} EM surrogate for the active block at node \(T\):
\[
\bar Q_t(\theta_t \mid \theta_t')
\;:=\;
\mathbb{E}_{X_{-t}\sim P_{\mathrm{tgt}}}\!\left[
\mathbb{E}_{\theta_t'}\!\left[\log p_{\theta_t}(T\mid X_{\pa(t)}) \,\middle|\, X_{-t}\right]
\right],
\]
with all frozen mechanisms \(\vartheta_{\setminus t}\) held fixed, where the inner expectation is taken over $T \sim p_{\theta_t'}(\cdot\mid X_{-t})$
(the E-step conditional under the current iterate), and the outer expectation is over
$X_{-t}\sim P_{\mathrm{tgt}}$.
.
The corresponding population EM operator is
\begin{equation}\label{eq:pop_em_operator_block}
F_t(\theta_t')
\;:=\;
\arg\max_{\theta_t\in\Omega_t}\;\bar Q_t(\theta_t \mid \theta_t'),
\end{equation}
where $\Omega_t$ is a feasible set (e.g., enforcing bounded log-variance; one convenient choice is
$\Omega_t=\{(b_t,\alpha_t): \|b_t\|_2\le B_{\max},\ \log\Delta_{\min}\le \alpha_t\le \log\Delta_{\max}\}$).

\medskip
\noindent\textbf{Block (partial) gradient-EM operator matching the algorithm.}
Our first-order method performs a single ascent update on the coefficient block \(b_t\) while optionally updating the variance in closed form. Accordingly, we define the population GEM mapping as the block-update operator
\begin{equation}\label{eq:pop_block_gem_operator}
G_t(\theta_t)
\;:=\;
\Bigl(
\,b_t + \eta\,\nabla_{b_t}\bar Q_t(\theta_t \mid \theta_t)\,,
\;\;\alpha_t^{+}(\theta_t)
\Bigr),
\end{equation}
where $\eta>0$ is a step size and $\alpha_t^{+}(\theta_t)$ is either (i) kept fixed, $\alpha_t^{+}(\theta_t)=\alpha_t$, or (ii) updated by a one-dimensional maximization of the surrogate given the updated $b$ (equivalently, the closed-form $\sigma_t^2$ update followed by $\alpha=\log\sigma^2$), with truncation to $\Omega_t$ if imposed. When constraints are enforced, interpret the mapping as followed by projection onto $\Omega_t$ (e.g., truncating $\alpha_t\in[\log\Delta_{\min},\log\Delta_{\max}]$).

\paragraph{Self-consistency and interiority.}
We assume the usual population self-consistency condition: \(\theta_t^*\) is a fixed point of the EM map,
equivalently \(\theta_t^*\in\arg\max_{\theta_t\in\Omega_t}\bar Q_t(\theta_t\mid\theta_t^*)\).
We also assume \(\theta_t^*\in\mathrm{int}(\Omega_t)\), so stationarity coincides with
\(\nabla_{\theta_t}\bar Q_t(\theta_t^*\mid\theta_t^*)=0\).

\paragraph{Neighborhood and norms.}
For \(r>0\), define the Euclidean ball
\(
\mathbb{B}(\theta_t^*;r):=\{\theta_t:\|\theta_t-\theta_t^*\|_2\le r\}.
\)
All conditions and results below are local to such a ball, which plays the role of a BWY \emph{basin of attraction}
around the population optimum.

\medskip
\noindent\textbf{BWY-style regularity conditions.}
We adopt the standard ``curvature + stability'' conditions used in modern EM analyses.

\begin{assumption}[Uniform local strong concavity and smoothness]\label{assump:local_sc_smooth}
There exist constants \(0<\lambda\le \mu\) and a radius \(r>0\) such that for every
\(\theta_t'\in\mathbb{B}(\theta_t^*;r)\), the function \(\theta_t \mapsto \bar Q_t(\theta_t \mid \theta_t')\)
is \(\lambda\)-strongly concave and \(\mu\)-smooth on \(\mathbb{B}(\theta_t^*;r)\); equivalently, for all
\(\theta_t,\theta_t'\in \mathbb{B}(\theta_t^*;r)\),
\[
-\mu I \;\preceq\; \nabla^2_{\theta_t}\bar Q_t(\theta_t \mid \theta_t') \;\preceq\; -\lambda I.
\]
\end{assumption}

\begin{assumption}[Gradient stability]\label{assump:grad_stability}
There exists \(\gamma\ge 0\) such that for all \(\theta_t,\theta_t' \in \mathbb{B}(\theta_t^*;r)\),
\[
\left\|\nabla_{\theta_t}\bar Q_t(\theta_t \mid \theta_t') - \nabla_{\theta_t}\bar Q_t(\theta_t \mid \theta_t^*)\right\|_2
\;\le\;
\gamma\,\|\theta_t' - \theta_t^*\|_2.
\]
\end{assumption}

\paragraph{Verifying Assumptions~\ref{assump:local_sc_smooth}--\ref{assump:grad_stability} in the linear-Gaussian active block.}
Assumptions~\ref{assump:local_sc_smooth}--\ref{assump:grad_stability} are standard in BWY-style EM analyses;
for our linear-Gaussian SEM they can be tied to explicit moment and conditioning quantities.
Assumption~\ref{assump:local_sc_smooth} follows from bounded-eigenvalue conditions on the parent covariance and an
interior log-variance constraint (Lemma~\ref{lem:lambda_mu}).
Assumption~\ref{assump:grad_stability} is governed by the sensitivity of the E-step moments of \(T\mid X_{-t}\) to
misspecification of \(\theta_t'\); via Louis' identity, smaller posterior uncertainty about \(T\) given \(X_{-t}\)
(e.g., due to informative observed descendants) reduces the missing-information term and yields a smaller stability
constant \(\gamma\). Proposition~\ref{prop:gamma_bound} gives a sufficient Lipschitz condition under which
Assumption~\ref{assump:grad_stability} holds with an explicit (data-dependent) upper bound on \(\gamma\).

\paragraph{When unlabeled target data cannot help.}
If \(T\) has no observed descendants (and more generally, if \(X_{-t}\) carries negligible information about \(T\)
under the target distribution), then the posterior \(p_{\theta_t}(T\mid X_{-t})\) is weakly informative and the
missing-information term can be large, leading to \(\gamma\) close to \(\lambda\) and thus slow or no contraction.
In this degenerate regime, unlabeled target samples cannot reliably identify a mechanism shift in
\(p(T\mid X_{\pa(t)})\), and significant adaptation gains should not be expected.

\medskip
\noindent\textbf{Contraction results.}

\begin{theorem}[Population contraction for EM and block gradient-EM]
\label{thm:population_contraction_param}
Suppose Assumptions~\ref{assump:local_sc_smooth}--\ref{assump:grad_stability} hold on
$\mathbb{B}(\theta_t^*;r)$ with $\gamma<\lambda$.
Write $\theta_t=(b_t,\alpha_t)$ and let $\nabla_1 \bar Q_t(\theta\mid\theta')$ denote the gradient
with respect to the \emph{first} argument $\theta$.

\begin{enumerate}[leftmargin=*]
\item \textbf{(Exact EM operator).}
For all $\theta_t\in \mathbb{B}(\theta_t^*;r)$,
\[
\|F_t(\theta_t)-\theta_t^*\|_2
\;\le\;
\kappa\,\|\theta_t-\theta_t^*\|_2,
\qquad
\kappa=\gamma/\lambda<1.
\]
Consequently, $F_t$ has a unique fixed point in $\mathbb{B}(\theta_t^*;r)$, and the iterates
$\theta_t^{(r+1)}=F_t(\theta_t^{(r)})$ converge geometrically whenever
$\theta_t^{(0)}\in\mathbb{B}(\theta_t^*;r)$.

\item \textbf{(Block first-order / gradient-EM coefficient update).}
Consider the coefficient update in the block map $G_t$ from \eqref{eq:pop_block_gem_operator}:
\[
b_t^+ \;=\; b_t+\eta\,\nabla_{b_t}\bar Q_t(\theta_t\mid\theta_t),
\qquad 0<\eta\le \frac{1}{\mu}.
\]
Then for all $\theta_t\in\mathbb{B}(\theta_t^*;r)$,
\[
\|b_t^+-b_t^*\|_2
\;\le\;
\bigl(1-\eta(\lambda-\gamma)\bigr)\,\|\theta_t-\theta_t^*\|_2.
\]
In particular, since $1/\mu \le 2/\mu$, this step-size choice is also compatible with the GEM ascent
condition (Lemma~\ref{lem:gem_ascent}) whenever the same smoothness constant is used.
Moreover, if the $\alpha_t$-update is \emph{contractive} with factor $\rho_\alpha<1$ on the ball (e.g.,
an exact EM update in $\alpha_t$ under the same $(\lambda,\mu,\gamma)$ framework, or any other contraction),
then the combined block map $G_t(\theta_t)=(b_t^+,\alpha_t^+)$ is contractive on $\mathbb{B}(\theta_t^*;r)$
(with contraction factor $\max\{1-\eta(\lambda-\gamma),\rho_\alpha\}$ under the product Euclidean norm).
\end{enumerate}
\end{theorem}

\begin{IEEEproof}[Proof sketch]
The EM-operator claim follows the BWY template as in \cite{Balakrishnan2017EM}.
For the block update, apply the same argument to the $b_t$-coordinate: write the $b$-update as gradient ascent on
$b_t\mapsto \bar Q_t\bigl((b_t,\alpha_t)\mid \theta_t^*\bigr)$ (with $\alpha_t$ fixed) plus an additive perturbation
controlled by Assumption~\ref{assump:grad_stability}.
If the variance/log-variance coordinate is held fixed, or if the resulting $\alpha$-update map is non-expansive after projection,
then composing it with the contractive $b_t$-update preserves contractivity of the combined mapping.
\end{IEEEproof}

\medskip
\noindent\textbf{Interpretation and connection to domain shift.}
Theorem~\ref{thm:population_contraction_param} is \emph{local}: it characterizes a basin \(\mathbb{B}(\theta_t^*;r)\) such that
initialization within this basin implies geometric convergence to the unique fixed point in that neighborhood.
In our setting, the source-fit initialization \(\theta_t^{(0)}=\theta_t^{(\mathrm{s})}\) is intended to land in (or near) this basin
when the local mechanism shift at \(T\) is not too large.

Both the curvature constants \((\lambda,\mu)\) and the stability constant \(\gamma\) depend on the target-domain distribution of observed
variables and the informativeness of the missingness pattern. In particular, when \(T\) has observed descendants (or other observed variables
whose distribution depends on \(T\)), the conditional moments \(T\mid X_{-t}\) are informative and \(\gamma\) is small; when \(T\) is nearly
conditionally independent of \(X_{-t}\), the E-step becomes weakly informative and \(\gamma\) can approach \(\lambda\), shrinking the basin and
slowing convergence. The contraction guarantee holds whenever the effective margin \(\lambda-\gamma>0\) remains positive in the target domain.

\subsection{EM Curvature Decomposition}\label{subsec:curvature}

This subsection clarifies the \emph{spectral-gap} intuition behind BWY-style contraction by recalling
Louis' classical \emph{missing-information principle}~\cite{Louis1982finding}.
Importantly, since our contraction analysis in Section~\ref{subsec:pop_contraction} is stated in the
\emph{SEM parameter space} (the active block at node \(t\)),
we present the curvature decomposition in a form consistent with that parameterization.
The purpose here is primarily interpretive: Louis' identity shows how latent/missing variables reduce
observed-data curvature, motivating why a positive ``complete-vs-missing'' information gap supports stable EM behavior.

\paragraph{Active parameterization (log-variance).}
For curvature statements that are well behaved in the variance coordinate, we parameterize the noise variance via
the log-variance \(\alpha_t:=\log\sigma_t^2\), and take the active block
\[
\theta_t=(b_t,\alpha_t)\in\mathbb{R}^{|\pa(t)|}\times\mathbb{R},
\qquad \text{so that } \sigma_t^2 = e^{\alpha_t}.
\]
All frozen mechanisms \(\vartheta_{\setminus t}\) are held fixed throughout.

\paragraph{Observed vs.\ complete information (Louis' identity).}
Fix \(\vartheta_{\setminus t}\) and let \(\theta_t=(b_t,\alpha_t)\) parameterize the local conditional \(p_{\theta_t}(T\mid X_{\pa(t)})\).
Let the observed- and complete-data \emph{negative} log-likelihood contributions for the \(\theta_t\)-dependent part be
\[
\ell_{\mathrm{obs}}(\theta_t)
\;:=\;
-\log p_{\theta_t}(X_{-t}),
\qquad\text{and}\qquad
\ell_{\mathrm{comp}}(\theta_t)
\;:=\;
-\log p_{\theta_t}(T\mid X_{\pa(t)}),
\]
where \(p_{\theta_t}(X_{-t})\) denotes the marginal induced by integrating out \(T\) under the SEM with \(\vartheta_{\setminus t}\) fixed.
(Equivalently, \(-\log p_{\theta_t}(T,X_{-t})\) differs from \(\ell_{\mathrm{comp}}(\theta_t)\) only by \(\theta_t\)-independent terms,
hence has the same \(\theta_t\)-derivatives.)

Louis' identity gives the pointwise curvature decomposition for the negative log-likelihood:
\begin{equation}\label{eq:louis_identity_param}
\nabla^2_{\theta_t}\,\ell_{\mathrm{obs}}(\theta_t)
\;=\;
\mathbb{E}_{\theta_t}\!\left[\nabla^2_{\theta_t}\,\ell_{\mathrm{comp}}(\theta_t)\,\middle|\,X_{-t}\right]
\;-\;
\mathrm{Var}_{\theta_t}\!\left(\nabla_{\theta_t}\,\ell_{\mathrm{comp}}(\theta_t)\,\middle|\,X_{-t}\right),
\end{equation}
where the conditional expectation/variance are taken under the model at \(\theta_t\).
Taking expectation over \(X_{-t}\) at \(\theta_t=\theta_t^*\) yields the population decomposition
\begin{equation}\label{eq:info_decomp_pop}
I_{\mathrm{obs}}
\;:=\;
\mathbb{E}\!\left[\nabla^2_{\theta_t}\,\ell_{\mathrm{obs}}(\theta_t^*)\right]
\;=\;
I_{\mathrm{comp}} - I_{\mathrm{miss}},
\end{equation}
with
\[
I_{\mathrm{comp}}
:=\mathbb{E}\!\left[\nabla^2_{\theta_t}\,\ell_{\mathrm{comp}}(\theta_t^*)\right],
\qquad
I_{\mathrm{miss}}
:=\mathbb{E}\!\left[\mathrm{Var}\!\left(\nabla_{\theta_t}\,\ell_{\mathrm{comp}}(\theta_t^*)\,\middle|\,X_{-t}\right)\right]\succeq 0.
\]
Thus, missingness of \(T\) can only \emph{reduce} curvature: \(I_{\mathrm{obs}}\preceq I_{\mathrm{comp}}\).

\medskip
\noindent\textbf{(a) Closed-form complete-data curvature for the linear-Gaussian mechanism at \(T\).}
Under the SEM, the conditional model at node \(T\) is
\[
T \;=\; b_t^\top X_{\pa(t)} + \varepsilon_t,
\qquad \varepsilon_t\sim \mathcal{N}(0,\sigma_t^2),\ \sigma_t^2=e^{\alpha_t}.
\]
Conditioned on \((T,X_{\pa(t)})\), the complete-data negative log-likelihood contribution is (up to constants)
\[
\ell_{\mathrm{comp}}(\theta_t)
=
\frac12\left[
\alpha_t + e^{-\alpha_t}\bigl(T-b_t^\top X_{\pa(t)}\bigr)^2
\right].
\]
Hence the complete-data curvature in \(b_t\) is
\begin{equation}\label{eq:Icomp_bt}
\nabla^2_{b_t}\,\ell_{\mathrm{comp}}(\theta_t^*)
=
e^{-\alpha_t^*}\,X_{\pa(t)}X_{\pa(t)}^\top
=
\frac{1}{\sigma_t^{2*}}\,X_{\pa(t)}X_{\pa(t)}^\top,
\qquad\Longrightarrow\qquad
I_{\mathrm{comp},b}
=
\frac{1}{\sigma_t^{2*}}\,\mathbb{E}\!\left[X_{\pa(t)}X_{\pa(t)}^\top\right].
\end{equation}
Analogous closed forms hold for the \(\alpha_t\) coordinate and cross-terms.

\begin{lemma}[Curvature constants for the active linear-Gaussian mechanism]
\label{lem:lambda_mu}
Assume $\alpha_t \in [\log\Delta_{\min},\log\Delta_{\max}]$ on $\mathbb{B}(\theta_t^*;r)$ with $\Delta_{\min}>0$, and assume
\[
m I \preceq \mathbb{E}[X_{\pa(t)}X_{\pa(t)}^\top] \preceq M I
\quad\text{for some }0<m\le M<\infty.
\]
Assume further that the residual second moment is locally bounded on $\mathbb{B}(\theta_t^*;r)$, i.e.,
\[
0 < v_{\min}\ \le\ \mathbb{E}\!\left[\bigl(T-b_t^\top X_{\pa(t)}\bigr)^2\right]\ \le\ v_{\max}<\infty
\quad \text{for all }\theta_t\in\mathbb{B}(\theta_t^*;r).
\]
Then, uniformly over $\theta_t'\in\mathbb{B}(\theta_t^*;r)$, the surrogate $\theta_t\mapsto \bar Q_t(\theta_t\mid\theta_t')$ is
\emph{blockwise} strongly concave/smooth with
\[
\lambda_b \ge \frac{m}{\Delta_{\max}},\quad \mu_b \le \frac{M}{\Delta_{\min}},
\qquad
\lambda_\alpha \ge \frac{1}{2}\frac{v_{\min}}{\Delta_{\max}},\quad
\mu_\alpha \le \frac{1}{2}\frac{v_{\max}}{\Delta_{\min}},
\]
for the $b_t$- and $\alpha_t$-coordinates, respectively.
Moreover, if the cross-curvature is controlled on the ball, e.g.,
\[
\sup_{\theta_t\in\mathbb{B}(\theta_t^*;r)}\bigl\|\nabla^2_{b_t\alpha_t}\bar Q_t(\theta_t\mid\theta_t')\bigr\|_2 \le \rho
\quad\text{with}\quad
\rho^2 < \lambda_b\lambda_\alpha,
\]
then Assumption~\ref{assump:local_sc_smooth} holds for the full block $\theta_t=(b_t,\alpha_t)$ with some $\lambda,\mu$ depending on
$(\lambda_b,\mu_b,\lambda_\alpha,\mu_\alpha,\rho)$ (e.g., by a Schur-complement bound).
\end{lemma}

\begin{proposition}[Sufficient condition for gradient stability]
\label{prop:gamma_bound}
Let \(\mu_{\theta'}(x_{-t})=\mathbb{E}_{\theta'}[T\mid X_{-t}=x_{-t}]\) denote the E-step conditional mean.
Suppose that on \(\mathbb{B}(\theta_t^*;r)\) there exists a measurable envelope \(L_\mu(x_{-t})\) with
\(\mathbb{E}[L_\mu(X_{-t})]<\infty\) such that for all \(\theta,\theta'\in \mathbb{B}(\theta_t^*;r)\),
\[
|\mu_{\theta'}(x_{-t})-\mu_{\theta}(x_{-t})|
\;\le\; L_\mu(x_{-t})\,\|\theta'-\theta\|_2
\quad \forall x_{-t}.
\]
Then Assumption~\ref{assump:grad_stability} holds with
\[
\gamma \;\le\; e^{-\alpha_{\min}}\;\mathbb{E}\!\left[\|X_{\pa(t)}\|_2\,L_\mu(X_{-t})\right]
\;\le\;
\frac{1}{\Delta_{\min}}\;\mathbb{E}\!\left[\|X_{\pa(t)}\|_2\,L_\mu(X_{-t})\right],
\]
where $\alpha_{\min}:=\log\Delta_{\min}$.
\end{proposition}

\begin{lemma}[Lipschitz conditional-mean map for one-missing-node Gaussian SEM]
\label{lem:lipschitz_mu}
Fix all frozen mechanisms \(\vartheta_{\setminus t}\) and consider the active block
\(\theta_t=(b_t,\alpha_t)\) in a neighborhood \(\mathbb{B}(\theta_t^\ast;r)\) with
\(\alpha_t\in[\log\Delta_{\min},\log\Delta_{\max}]\).
Let \(K(\theta_t)\) denote the implied precision matrix under the SEM parameters (with \(\vartheta_{\setminus t}\) fixed).
For the single missing coordinate \(T=X_t\), the Gaussian conditional mean admits the precision form
\[
\mu_{\theta_t}(x_{-t})
= m_t(\theta_t) - K_{tt}(\theta_t)^{-1}K_{t,-t}(\theta_t)\bigl(x_{-t}-m_{-t}(\theta_t)\bigr).
\]
Assume:
(i) \(K_{tt}(\theta_t)\ge c_K>0\) for all \(\theta_t\in \mathbb{B}(\theta_t^\ast;r)\), and
(ii) the map \(\theta_t\mapsto (m(\theta_t),K_{tt}(\theta_t),K_{t,-t}(\theta_t))\) is continuously differentiable on \(\mathbb{B}(\theta_t^\ast;r)\) with
\[
\sup_{\theta_t\in \mathbb{B}(\theta_t^\ast;r)}
\left\|\nabla_{\theta_t} m(\theta_t)\right\|_{\mathrm{op}} \le C_m,\qquad
\sup_{\theta_t\in \mathbb{B}(\theta_t^\ast;r)}
\left\|\nabla_{\theta_t}\!\left(K_{tt}(\theta_t)^{-1}K_{t,-t}(\theta_t)\right)\right\|_{\mathrm{op}} \le C_K.
\]
Then for all \(\theta_t,\theta'_t\in \mathbb{B}(\theta_t^\ast;r)\) and all \(x_{-t}\),
\[
\bigl|\mu_{\theta'_t}(x_{-t})-\mu_{\theta_t}(x_{-t})\bigr|
\;\le\; L_\mu(x_{-t}) \,\|\theta'_t-\theta_t\|_2,
\qquad
L_\mu(x_{-t}) := C_m + C_K\,\|x_{-t}-m_{-t}(\theta_t^\ast)\|_2.
\]
In particular, if \(X_{-t}\) has finite second moment under the target distribution (e.g., is sub-Gaussian),
then \(\mathbb{E}[L_\mu(X_{-t})]<\infty\) and the condition of Proposition~\ref{prop:gamma_bound} holds.
\end{lemma}

\subsection{High-Probability Sample-Level Concentration and Final Error Bound}\label{subsec:sample_concentration}

We now translate the population contraction result of Section~\ref{subsec:pop_contraction} into a finite-sample guarantee for our domain-adaptive (gradient-)EM updates on the active mechanism at node \(t\).
Consistent with Section~\ref{subsec:pop_contraction}--\ref{subsec:curvature}, we parameterize the active block as
\[
\theta_t=(b_t,\alpha_t),\qquad \alpha_t:=\log \sigma_t^2,
\]
keeping all source-invariant mechanisms fixed and analyzing the stochastic error induced by estimating the target-domain block-GEM update from \(N_{\mathrm{t}}\) unlabeled target samples.

\paragraph{Sample vs.\ population operators.}
Let \(G_t\) denote the \emph{population} block-GEM mapping on the active block (cf.\ \eqref{eq:pop_block_gem_operator}), and let \(\widehat{G}_t\) denote its finite-sample counterpart obtained by replacing population expectations with empirical averages (cf.\ Section~\ref{subsec:gradient_em}--\ref{subsec:sample_em}).
Concretely, \(\widehat{G}_t\) uses the sample parent moment \(\widehat{M}_{\pa(t)}\) and the imputed cross-moment \(\widehat{v}_t^{(r)}\), performs the same gradient-ascent step on the coefficient block \(b_t\), and uses the same choice of variance/log-variance update (kept fixed or updated in closed form with truncation).
We suppress the dependence on frozen mechanisms in the notation and treat them as fixed for the main argument.

\paragraph{Uniform deviation bound.}
To control the discrepancy \(\widehat{G}_t-G_t\) uniformly over the local basin, assume:
(i) \(X_{\pa(t)}\) is sub-Gaussian under the target distribution, and
(ii) the conditional-moment map \(x_{-t}\mapsto \mu_t(x_{-t};\theta_t)=\mathbb{E}_{\theta_t}[T\mid X_{-t}=x_{-t}]\) is uniformly Lipschitz in \(\theta_t\) over \(\mathbb{B}(\theta_t^*;r)\) with an envelope ensuring sub-Gaussian (or sub-exponential) tails for the random vectors \(X_{\pa(t)}\mu_t(X_{-t};\theta_t)\).
Under these standard regularity conditions, empirical-process concentration yields the uniform high-probability bound
\begin{equation}\label{eq:operator_dev_bound}
\sup_{\theta_t\in \mathbb{B}(\theta_t^*;r)}
\bigl\|\widehat{G}_t(\theta_t) - G_t(\theta_t)\bigr\|_2
\;\le\;
\delta_{N_{\mathrm{t}}},
\end{equation}
with probability at least \(1-\xi\), where
\[
\delta_{N_{\mathrm{t}}}
=
O\!\left(\sqrt{\frac{d_t+\log(1/\xi)}{N_{\mathrm{t}}}}\right),
\qquad
d_t:=\dim(b_t)+1.
\]
Here \(\dim(b_t)=|\pa(t)|\) without an intercept and \(\dim(b_t)=|\pa(t)|+1\) with an intercept, and the additional \(+1\) accounts for the log-variance parameter \(\alpha_t\).

\paragraph{Finite-sample convergence to a statistical neighborhood.}
Assume the population mapping \(G_t\) is \(\kappa\)-contractive on \(\mathbb{B}(\theta_t^*;r)\), i.e.,
\begin{equation}\label{eq:pop_contraction_kappa}
\|G_t(\theta_t)-\theta_t^*\|_2\le \kappa\|\theta_t-\theta_t^*\|_2,
\qquad \forall\,\theta_t\in\mathbb{B}(\theta_t^*;r),
\end{equation}
with \(0\le \kappa<1\) and \(G_t(\theta_t^*)=\theta_t^*\).
On the event \eqref{eq:operator_dev_bound}, the sample iterates \(\theta_t^{(r+1)}=\widehat{G}_t(\theta_t^{(r)})\) satisfy
\[
\|\theta_t^{(r+1)}-\theta_t^*\|_2
\le
\|\widehat{G}_t(\theta_t^{(r)})-G_t(\theta_t^{(r)})\|_2
+
\|G_t(\theta_t^{(r)})-\theta_t^*\|_2
\le
\delta_{N_{\mathrm{t}}}+\kappa\|\theta_t^{(r)}-\theta_t^*\|_2.
\]
Unrolling yields, for all \(r\ge 0\),
\begin{equation}\label{eq:final_error_bound}
\|\theta_t^{(r)}-\theta_t^*\|_2
\;\le\;
\kappa^{\,r}\,\|\theta_t^{(0)}-\theta_t^*\|_2
\;+\;
\frac{\delta_{N_{\mathrm{t}}}}{1-\kappa}.
\end{equation}

\paragraph{Basin invariance.}
Since the contraction in \eqref{eq:pop_contraction_kappa} is local, we require the iterates remain in \(\mathbb{B}(\theta_t^*;r)\).
A sufficient condition is that
\[
\|\theta_t^{(0)}-\theta_t^*\|_2 \;\le\; r-\frac{\delta_{N_{\mathrm{t}}}}{1-\kappa},
\qquad\text{and}\qquad
\frac{\delta_{N_{\mathrm{t}}}}{1-\kappa}\;<\; r,
\]
in which case \eqref{eq:final_error_bound} implies \(\|\theta_t^{(r)}-\theta_t^*\|_2\le r\) for all \(r\).

\paragraph{Remark (source estimation error).}
The bound above conditions on the frozen (source-invariant) mechanisms and treats them as fixed.
In practice, these mechanisms are estimated from \(N_{\mathrm{s}}\) source samples; under standard sub-Gaussian assumptions and a consistent DAG fit,
\(\|\vartheta_{\setminus t}^{(\mathrm{s})}-\vartheta_{\setminus t}^{*}\|=O_{\mathbb{P}}(N_{\mathrm{s}}^{-1/2})\) (up to dimension/log factors) in an appropriate Euclidean/operator norm.
Local Lipschitz dependence of the E-step moments on the frozen block then contributes an additional additive term of order \(O_{\mathbb{P}}(N_{\mathrm{s}}^{-1/2})\) to \eqref{eq:operator_dev_bound}, and hence to the statistical floor in \eqref{eq:final_error_bound}.

\paragraph{Implication for target imputation.}
Let \(\widehat{T}_{\theta_t}(x_{-t}) := \mathbb{E}_{\theta_t}[T\mid X_{-t}=x_{-t}]\) denote the model-based imputer (conditional mean).
Under the same regularity conditions used to establish \eqref{eq:operator_dev_bound}, this imputation map is locally Lipschitz in \(\theta_t\) on \(\mathbb{B}(\theta_t^*;r)\); that is, there exists a measurable function \(L_{\mathrm{imp}}(X_{-t})\) with \(\mathbb{E}[L_{\mathrm{imp}}(X_{-t})]<\infty\) such that
\[
\bigl|\widehat{T}_{\theta_t}(X_{-t})-\widehat{T}_{\theta_t^*}(X_{-t})\bigr|
\;\le\;
L_{\mathrm{imp}}(X_{-t})\,\|\theta_t-\theta_t^*\|_2,
\qquad \forall\,\theta_t\in\mathbb{B}(\theta_t^*;r).
\]
Consequently, combining this Lipschitz property with \eqref{eq:final_error_bound} yields a high-probability statistical guarantee for imputation error: it decays geometrically in the iteration index \(r\) up to a statistical floor of order
\(O\!\bigl(\delta_{N_{\mathrm{t}}}/(1-\kappa)\bigr)\) (and an additional \(O_{\mathbb{P}}(N_{\mathrm{s}}^{-1/2})\) floor from estimating frozen mechanisms), up to logarithmic factors.

\subsection{Other EM Variants with Geometric-Rate Guarantees}\label{subsec:other_variants}

Our main algorithm uses a first-order (gradient) M-step for scalability on the \emph{active} mechanism at \(T\).
It is natural to ask whether other EM-family updates also admit BWY-style \emph{local} geometric convergence in our Gaussian DAG setting when we (i) freeze all source-invariant mechanisms and (ii) restrict optimization to the shifted block
\[
\theta_t=(b_t,\alpha_t),\qquad \alpha_t:=\log\sigma_t^2.
\]
Under the same local curvature and stability assumptions used in Section~\ref{subsec:pop_contraction}--\ref{subsec:curvature}, several classical variants inherit analogous local contraction guarantees.
Below we summarize three representative examples and contrast their per-iteration costs in terms of the active-block dimension
\[
d:=\dim(b_t)+1,
\]
where \(\dim(b_t)=|\pa(t)|\) without an intercept and \(\dim(b_t)=|\pa(t)|+1\) with an intercept, and the additional \(+1\) accounts for \(\alpha_t\).

\paragraph{Exact EM (restricted to the active block).}
Consider the exact population EM operator
\(
F_t(\theta_t')=\arg\max_{\theta_t\in\Omega_t}\bar Q_t(\theta_t\mid \theta_t')
\)
with all other mechanisms frozen.
Under Assumptions~\ref{assump:local_sc_smooth}--\ref{assump:grad_stability}, \(F_t\) is contractive on \(\mathbb{B}(\theta_t^*;r)\) with factor \(\kappa=\gamma/\lambda<1\) (Theorem~\ref{thm:population_contraction_param}).
At the sample level, this corresponds to an ECM-style update \cite{meng1993maximum} that performs a \emph{closed-form} regression update for \(b_t\) (and a scalar closed-form update for \(\alpha_t\), equivalently for \(\sigma_t^2\)) using the imputed sufficient statistics.
Computationally, the dominant linear algebra is solving a \(\dim(b_t)\times \dim(b_t)\) linear system for \(b_t\), yielding per-iteration cost \(O(\dim(b_t)^3)\) in general (or \(O(\dim(b_t)^2)\) per iteration if a factorization of \(\widehat M_{\pa(t)}\) is cached and reused across iterations).

\paragraph{ECME (observed-likelihood maximization for selected coordinates).}
ECME \cite{liu1994ecme} replaces some conditional maximizations of the surrogate by direct maximization of the observed-data likelihood.
In our setting, one convenient instance keeps the E-step unchanged, updates \(b_t\) by the completed-data regression, and updates \(\alpha_t\) (equivalently \(\sigma_t^2\)) by maximizing the \emph{target observed-data} likelihood with respect to that coordinate (holding the remaining blocks fixed).
Under the same local curvature/stability conditions and standard regularity for the observed-likelihood coordinate update, the resulting mapping is locally contractive on \(\mathbb{B}(\theta_t^*;r)\).
Computationally, this update remains dominated by the \(\dim(b_t)\times \dim(b_t)\) linear solve, hence is \(O(\dim(b_t)^3)\) per iteration in the worst case.

\paragraph{PX-EM (parameter expansion; applicability outline).}
PX-EM \cite{liu1998parameter} introduces an expanded parameterization together with a deterministic reduction mapping back to the original parameter space, often improving practical convergence by reducing the effective fraction of missing information.
In our Gaussian DAG setting, a natural expansion can be restricted to the active mechanism at \(T\) (e.g., a scale expansion acting on \((b_t,\sigma_t^2)\) in the expanded space, followed by a smooth reduction map back to \((b_t,\alpha_t)\)).
Under additional regularity ensuring that the expansion--reduction mapping is smooth and locally invertible in a neighborhood of \(\theta_t^*\), one can apply the same local contraction logic to the reduced operator on \(\theta_t\).
A complete proof in our setting requires (i) verifying local invertibility of the reduction map and (ii) bounding the Jacobian of the reduced update to control the induced contraction factor; we outline these steps in the supplementary material.

\medskip
\noindent\textbf{Remark.}
All guarantees above are \emph{local}: they require initialization in a basin \(\mathbb{B}(\theta_t^*;r)\) and a positive complete-vs.-missing information gap (Section~\ref{subsec:curvature}).
The key modeling choice enabling such results for domain adaptation is the restriction to a \emph{local mechanism shift at \(T\)} and the corresponding block-restricted updates; when additional mechanisms shift, the active block expands and the same contraction framework can be applied provided the corresponding curvature and stability conditions continue to hold.

\section{Experimental Results}\label{sec:experiments}

We evaluate the proposed \emph{DAG-aware first-order (gradient) EM} procedure for imputing a designated target variable \(T\) that is systematically missing in the deployment (target) domain. Throughout, we assume a known Gaussian causal DAG and compare against (i) a \emph{fit-on-source} Gaussian Bayesian network baseline and (ii) a \emph{Kiiveri-style} EM implementation for Gaussian covariance-structure models with one latent node.
Our study includes (a) controlled simulations, where the ground-truth shift mechanism is known, (b) a higher-dimensional benchmark on the 64-node MAGIC-IRRI network, and (c) a real-data case study on single-cell signaling measurements (Sachs et al.). 

\noindent\textbf{Why we do not include importance weighting (IW).}
Importance weighting is designed for \emph{covariate shift}, where the conditional mechanism $p(T\mid X)$ remains invariant while $p(X)$ changes.
In our main setting of \emph{local mechanism shift at $T$}, the conditional $p_{\mathrm{tgt}}(T\mid X_{\pa(t)})$ itself changes across domains.
Consequently, reweighting labeled source samples alone---which are generated under the \emph{source} mechanism---cannot, by itself, identify the parameters of the \emph{target} mechanism.
Our approach instead adapts the active mechanism parameters by leveraging unlabeled target structure through the DAG, in particular the covariance information carried by observed descendants of $T$ when $T$ is systematically missing.

All experiments were run on a Windows workstation equipped with a 12th Gen Intel(R) Core(TM) i9-12900H 2.50\,GHz CPU.
Code to reproduce the experiments is available at \url{https://github.com/majavid/ICDM2025}.

\paragraph*{Evaluation protocol}
In all experiments, $T$ is hidden only in the target domain during training, but retained for evaluation.
We report MAE, RMSE, and $R^2$ on the imputed $T$.
Unless stated otherwise, MAE and RMSE are computed after z-score standardization of $T$ (using the source-domain mean and standard deviation), so errors are reported in standard-deviation units.

\subsection{Simulated Experiments}\label{subsec:sim_experiments}

\paragraph{Seven-node SEM and shift design.}
We revisit the motivating seven-node linear-Gaussian SEM from Section~\ref{sec:intro}, in which context variables \(C_1,C_2\) drive intermediate nodes \(Z\) and \(X\), which together with \(C_1\) determine the target node \(T\), and \(T\) influences outcomes \(P\) and \(Y\).
We generate a fully observed \emph{source} dataset and a \emph{target} dataset in which \(T\) is completely unobserved during training.

To align with our problem formulation, we consider two shift classes:
\begin{itemize}[leftmargin=*]
  \item \textbf{Covariate/root shift:} we modify the marginal distribution of a context/root variable (e.g., a large change in the mean/variance of $C_2$),
while keeping all \emph{non-root} conditional mechanisms $P(X_k\mid X_{\pa(k)})$ invariant.
  \item \textbf{Local mechanism shift at \(T\):} we modify only the conditional mechanism generating \(T\), i.e., we change the coefficients and/or intercept in the structural equation for \(T\) while keeping all other conditionals invariant (cf.\ Section~\ref{sec:problem}).\footnote{Changing only \(\Var(\varepsilon_T)\) does not affect \(\mathbb{E}[T\mid X_{\pa(t)}]\) in a linear-Gaussian SEM; thus mean-imputation improvements under ``target shift'' require a mechanism change in \(P(T\mid X_{\pa(t)})\).}
\end{itemize}

\paragraph{Methods compared.}
We compare:
(i) \textbf{Baseline (Fit-on-Source)}: fit the source-domain Gaussian BN/SEM parameters and impute \(T\) in the target using the source estimate without adaptation;
(ii) \textbf{Kiiveri EM}: a covariance-structure EM procedure treating \(T\) as latent in the target;
(iii) \textbf{1st-order EM (ours)}: our domain-adaptive gradient-EM update on the active mechanism at \(T\), freezing source-invariant mechanisms and iterating EM updates until convergence (typically a small number of iterations; see supplement).
In the seven-node SEM, we impute \(T\) from observed variables in \(X_{-t}\); When conditioning on descendants/correlated variables (i.e., using $X_{-t}$ beyond parents), updating the shifted \emph{root} marginals using unlabeled target data can improve the target covariance used in $\mathbb{E}[T\mid X_{-t}]$; this is the sense in which adaptation can help in our covariate/root shift setting.

\paragraph{Results.}
Table~\ref{tab:results10} reports average performance over 10 repetitions.
The fit-on-source baseline remains accurate under covariate shift but degrades substantially under local mechanism shift at \(T\), consistent with a mismatch in the conditional \(P(T\mid X_{\pa(t)})\).
Our 1st-order EM achieves consistently low MAE/RMSE and near-perfect \(R^2\) under both shift types, indicating that adapting only the shifted mechanism can recover near-oracle imputation accuracy.
In our implementation, the Kiiveri EM baseline often converges to numerically unstable or degenerate solutions under large shifts.

\begin{table}[!ht]
  \centering
  \caption{Average target-domain imputation error under covariate shift and local mechanism shift at \(T\) (10 repeats).}
  \label{tab:results10}
  \resizebox{.75\columnwidth}{!}{%
  \begin{tabular}{llccc}
    \toprule
    \textbf{Shift scenario} & \textbf{Method} & \textbf{MAE} & \textbf{RMSE} & \(\mathbf{R}^2\) \\
    \midrule
    \multirow{3}{*}{Covariate shift} 
      & Baseline (Fit-on-Source) & 0.7935 & 0.9945 & 0.9981 \\
      & Kiiveri EM               & 45.1882 & 45.1973 & –2.9821 \\
      & 1st-order EM             & \textbf{0.3299} & \textbf{0.4145} & \textbf{0.9997} \\
    \midrule
    \multirow{3}{*}{Mechanism shift at \(T\)}
      & Baseline (Fit-on-Source) & 6.0107 & 6.3333 & 0.9473 \\
      & Kiiveri EM               & 70.8294 & 72.1688 & –5.8331 \\
      & 1st-order EM             & \textbf{0.9312} & \textbf{1.0577} & \textbf{0.9985} \\
    \bottomrule
  \end{tabular}%
  }
\end{table}

\subsection{MAGIC-IRRI: High-Dimensional Gaussian DAG under Strong Interventions}\label{subsec:magic_irri}

We next evaluate on the 64-node MAGIC-IRRI Gaussian Bayesian network from Scutari (ICQG 2016), available via the BN repository.\footnote{Network structure and data: \url{https://www.bnlearn.com/bnrepository/}.}
We treat the published network as the causal DAG \(\mathcal{G}\), designate \texttt{HT} as the systematically missing target variable in the deployment domain, and simulate a shifted target domain by applying large marginal interventions to three observed variables:
\begin{itemize}[leftmargin=*]
  \item \textbf{G4156}: from \(N(0.7636,\,0.9721^2)\) to \(N(1.5,\,2.0^2)\),
  \item \textbf{G4573}: from \(N(0.1196,\,0.4744^2)\) to \(N(1.0,\,1.0^2)\),
  \item \textbf{G1533}: from \(N(0.8004,\,0.9803^2)\) to \(N(0,\,3.0^2)\).
\end{itemize}
These interventions change the marginal distribution of observed covariates and propagate through the DAG, inducing a substantial distribution shift in the joint law of \(X_{-t}\).
Although the \emph{structural mechanisms} may remain unchanged away from the interventions, the posterior \(\mathbb{E}[T\mid X_{-t}]\) depends on the target-domain covariance; consequently, imputing \(T\) using a source-fitted covariance can be strongly miscalibrated when conditioning on descendants and other correlated variables.

Table~\ref{tab:magic_irri_results} summarizes the imputation results.
The fit-on-source baseline performs poorly under these strong shifts (negative \(R^2\)), and Kiiveri EM provides only marginal improvement in this regime.
In contrast, our 1st-order EM substantially reduces MAE/RMSE and achieves a positive \(R^2\), indicating that a lightweight domain-adaptive covariance/mechanism correction can recover meaningful predictive power even in a high-dimensional, heavily perturbed Gaussian DAG.

\begin{table}[ht]
  \centering
  \caption{Imputation performance on the MAGIC-IRRI DAG under strong marginal interventions (target: \texttt{HT}).}
  \label{tab:magic_irri_results}
  \resizebox{.75\columnwidth}{!}{%
  \begin{tabular}{lccc}
    \toprule
    \textbf{Method} & \textbf{MAE} & \textbf{RMSE} & \(\mathbf{R}^2\) \\
    \midrule
    Baseline (Fit-on-Source) 
      & 9.3827  & 11.1872  & –0.0957 \\
    Kiiveri EM              
      & 8.8479  & 11.0771  & –0.0743 \\
    1st-order EM            
      & \textbf{5.5834}   & \textbf{7.0277}  & \textbf{0.5676} \\
    \bottomrule
  \end{tabular}%
  }
\end{table}

\begin{figure}[ht]
  \centering
  \subfloat[Baseline (Fit-on-Source)\label{fig:bn-cov-magic}]{
    \includegraphics[width=0.32\linewidth]{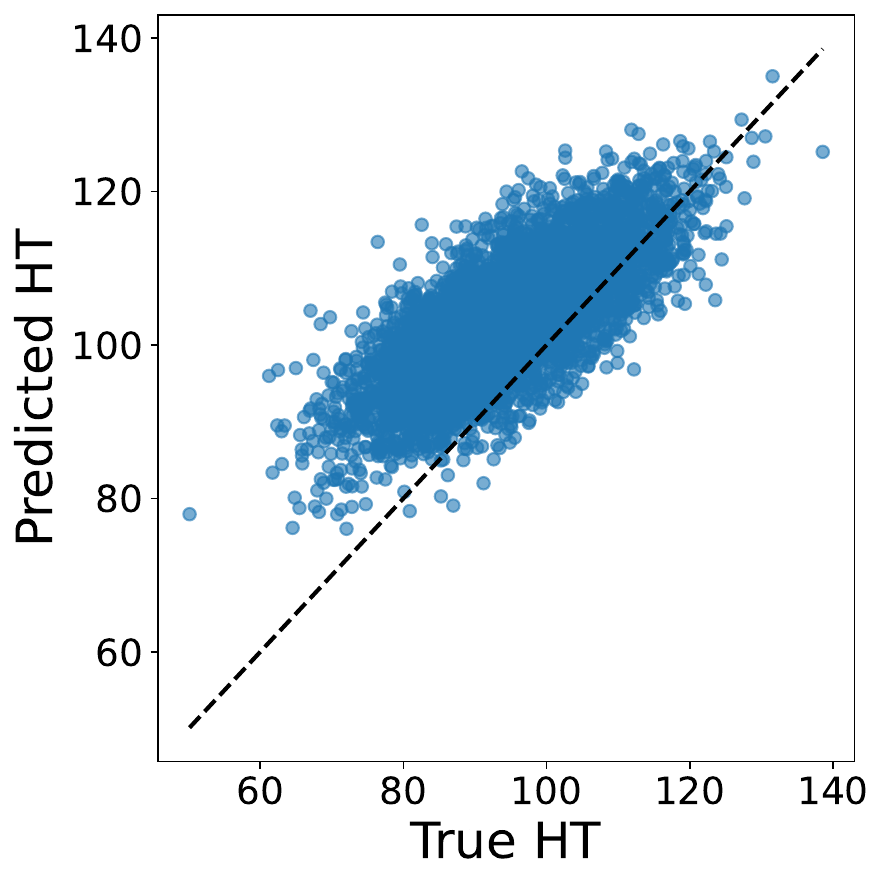}
  }%
  \subfloat[Kiiveri EM\label{fig:kiiv-cov-magic}]{
    \includegraphics[width=0.32\linewidth]{images/KiiveriCov.pdf}
  }%
  \subfloat[1st-order EM\label{fig:fo-cov-magic}]{
    \includegraphics[width=0.32\linewidth]{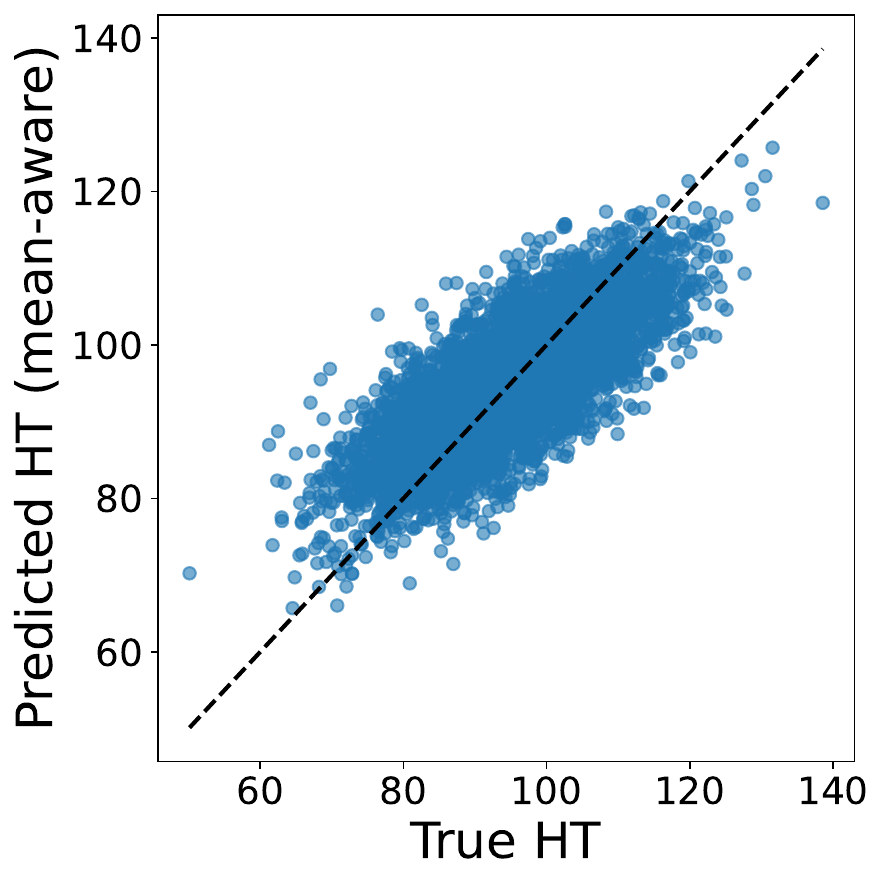}
  }
  \caption{True versus predicted \texttt{HT} under strong interventions for three methods:
  (a) fit-on-source baseline, (b) Kiiveri EM, (c) our 1st-order EM.}
  \label{fig:ht_scatter_all}
\end{figure}

Figure~\ref{fig:ht_scatter_all} visualizes the same comparison.
The fit-on-source baseline exhibits substantial bias and dispersion, consistent with negative \(R^2\).
Kiiveri EM shows signs of instability under this regime (predictions collapsing toward a narrow range).
Our 1st-order EM yields markedly better calibration around the \(y=x\) line, consistent with the improved error metrics.

\subsection{Real-Data Experiment: Single-Cell Signaling (Sachs et al.)}\label{subsec:sachs}

Finally, we evaluate on the single-cell flow cytometry dataset of Sachs et al.~\cite{sachs2005causal}, which measures phosphorylated signaling proteins in human primary CD4$^+$ T cells under multiple experimental conditions.
This dataset is a stringent test for transfer because interventions induce pronounced distribution shifts across conditions.
We designate the anti-CD3/CD28 stimulation condition (853 cells) as the source domain and the PMA stimulation condition (913 cells) as the target domain, and we treat each of ten proteins (Raf, Mek, Plcg, PIP$_2$, PIP$_3$, Erk, Akt, PKA, P38, Jnk) in turn as the target \(T\) that is systematically hidden in the target domain during training.

Table~\ref{tab:sachs_multi} reports target-domain imputation accuracy.
We observe strong gains for several proteins (notably Raf and Erk), indicating that the proposed procedure can leverage source information together with the target-domain observed distribution to improve posterior imputation under intervention-induced shift.
At the same time, for certain targets (e.g., Mek, PKA), performance deteriorates, yielding negative \(R^2\).
Such cases likely reflect violations of the modeling assumptions (non-Gaussianity, hidden confounding, and feedback), as well as mechanism changes that are not well captured by a linear-Gaussian DAG.
These results therefore provide both validation (where the assumptions are approximately met) and a clear motivation for robust extensions beyond linear-Gaussian DAGs.

\begin{table}[!t]
\caption{Imputation performance on the Sachs et al.\ data under domain shift (source: CD3/CD28; target: PMA).}
\label{tab:sachs_multi}
\centering
\small
\begin{tabular}{lrrr r}
\toprule
Target Variable & Method & MAE & RMSE & \(R^2\) \\
\midrule
Raf & Baseline (Fit-on-Source) & 0.6908 & 1.0015 & -0.0041 \\
    & 1st-order EM             & \textbf{0.4132} & \textbf{0.6393} & \textbf{0.5908} \\
    & Kiiveri EM               & 0.5132 & 0.8324 & 0.3064 \\
\midrule
Mek & Baseline (Fit-on-Source) & \textbf{0.3933} & \textbf{0.6383} & \textbf{0.5922} \\
    & 1st-order EM             & 0.7140 & 1.0701 & -0.1464 \\
    & Kiiveri EM               & 0.7143 & 1.0707 & -0.1476 \\
\midrule
Plcg & Baseline (Fit-on-Source) & 0.6529 & 0.9995 & -0.0000 \\
     & 1st-order EM             & 0.5858 & 0.9008 & 0.1876 \\
     & Kiiveri EM               & \textbf{0.5851} & \textbf{0.8998} & \textbf{0.1894} \\
\midrule
PIP2 & Baseline (Fit-on-Source) & 0.6156 & 0.8254 & 0.3180 \\
     & 1st-order EM             & 0.6156 & 0.8254 & 0.3180 \\
     & Kiiveri EM               & 0.6165 & 0.8263 & 0.3165 \\
\midrule
PIP3 & Baseline (Fit-on-Source) & 0.5240 & 0.9280 & 0.1378 \\
     & 1st-order EM             & 0.3809 & 0.8106 & 0.3422 \\
     & Kiiveri EM               & \textbf{0.3731} & \textbf{0.8049} & \textbf{0.3515} \\
\midrule
Erk  & Baseline (Fit-on-Source) & 0.5884 & 0.8379 & 0.2971 \\
     & 1st-order EM             & \textbf{0.1817} & \textbf{0.2855} & \textbf{0.9184} \\
     & Kiiveri EM               & 4.7827 & 6.8503 & -45.9786 \\
\midrule
Akt  & Baseline (Fit-on-Source) & 0.1744 & 0.2756 & 0.9240 \\
     & 1st-order EM             & 0.1744 & 0.2756 & 0.9240 \\
     & Kiiveri EM               & \textbf{0.1728} & \textbf{0.2742} & \textbf{0.9247} \\
\midrule
PKA  & Baseline (Fit-on-Source) & \textbf{0.6810} & \textbf{0.9992} & 0.0006 \\
     & 1st-order EM             & 0.7594 & 1.0982 & -0.2073 \\
     & Kiiveri EM               & 0.8031 & 1.1795 & -0.3927 \\
\midrule
P38  & Baseline (Fit-on-Source) & 0.2897 & 0.4543 & 0.7934 \\
     & 1st-order EM             & 0.2897 & 0.4543 & 0.7934 \\
     & Kiiveri EM               & 0.2897 & 0.4543 & 0.7934 \\
\midrule
Jnk  & Baseline (Fit-on-Source) & 0.6621 & 1.1185 & -0.2525 \\
     & 1st-order EM             & 0.6621 & 1.1185 & -0.2525 \\
     & Kiiveri EM               & 0.6620 & 1.1185 & -0.2525 \\
\bottomrule
\end{tabular}
\end{table}

\section{Conclusion}

We studied the problem of imputing a designated target variable \(T\) that is systematically missing in a shifted deployment domain, leveraging a known Gaussian causal DAG learned from fully observed source data. We proposed a DAG-aware first-order (gradient) EM framework that performs a \emph{block-local} update: it freezes source-invariant mechanisms and adapts only the conditional mechanism of \(T\) using unlabeled target observations and the covariance information propagated through observed descendants. Under BWY-style local regularity conditions (strong concavity/smoothness and a complete--vs.--missing information spectral gap), we established local geometric convergence of the population operator and high-probability sample-level convergence to a statistical neighborhood, yielding finite-sample guarantees for target imputation.

Empirically, across a synthetic seven-node SEM, the 64-node MAGIC-IRRI network, and the Sachs single-cell signaling data, the proposed method consistently improves target-domain imputation over a fit-on-source Bayesian network and a Kiiveri-style EM baseline, especially under pronounced shifts. Importantly, our updates operate in the DAG parameter space and require only local sufficient statistics, making the procedure scalable in high-dimensional graphs.

Several directions remain open. First, extending the framework from a single systematically missing node to \emph{multiple} missing/latent nodes will require blockwise E-steps and careful control of the resulting missing-information fraction. Second, relaxing causal sufficiency and accommodating latent confounding or selection bias (e.g., via ADMGs/ancestral graphs) would broaden applicability, but demands new conditional-moment computations and corresponding contraction analyses. Finally, developing guarantees under \emph{model misspecification}---including nonlinear mechanisms, feedback effects, or non-Gaussian noise as suggested by some signaling targets---is an important step toward robust deployment in complex scientific systems.

\appendix
\section*{Proof of Main Theoretical Results}
\begin{proof}[Proof of Lemma \ref{lem:gem_ascent}]
Fix $\vartheta^{(r)}$ and hold $\sigma_t^2$ fixed at $(\sigma_t^2)^{(r)}$. Conditioned on $\vartheta^{(r)}$, the E-step moments
$\{\mu_t^{(r)}(x_{-t}^{(i)}),V_t^{(r)}\}_{i=1}^n$ are treated as constants in the M-step surrogate.
In the Gaussian SEM, the only part of $\widehat Q(\vartheta\mid \vartheta^{(r)})$ that depends on $b_t$ is the quadratic regression term
induced by the conditional $T\mid X_{\pa(t)}$. Using the imputed sufficient statistics in \eqref{eq:imputed_stats}, we can write, up to an
additive constant independent of $b_t$,
\begin{equation}\label{eq:Qhat_quadratic_bt}
\widehat Q(b_t\mid \vartheta^{(r)})
=
\frac{1}{(\sigma_t^2)^{(r)}}\left(
b_t^\top \widehat v_t^{(r)} - \frac12\, b_t^\top \widehat M_{\pa(t)}\, b_t
\right)
+\text{const},
\end{equation}
where
$\widehat M_{\pa(t)}=\frac1n\sum_{i=1}^n x_{\pa(t)}^{(i)}x_{\pa(t)}^{(i)\top}$ and
$\widehat v_t^{(r)}=\frac1n\sum_{i=1}^n x_{\pa(t)}^{(i)}\,\mu_t^{(r)}(x_{-t}^{(i)})$ as in \eqref{eq:imputed_stats}.

Differentiating \eqref{eq:Qhat_quadratic_bt} yields the gradient in \eqref{eq:grad_bt}:
\[
\nabla_{b_t}\widehat Q(b_t\mid \vartheta^{(r)})
=
\frac{1}{(\sigma_t^2)^{(r)}}\left(\widehat v_t^{(r)}-\widehat M_{\pa(t)}b_t\right),
\]
and the Hessian is the constant matrix
\[
\nabla^2_{b_t}\widehat Q(b_t\mid \vartheta^{(r)})
=
-\frac{1}{(\sigma_t^2)^{(r)}}\,\widehat M_{\pa(t)}.
\]
Since $\widehat M_{\pa(t)}\succeq 0$, the Hessian is negative semidefinite, hence $\widehat Q(\cdot\mid \vartheta^{(r)})$ is concave in $b_t$.
Moreover, the gradient is Lipschitz with constant equal to the operator norm of the Hessian:
\[
\|\nabla_{b_t}\widehat Q(b)-\nabla_{b_t}\widehat Q(b')\|_2
\le
\left\|\nabla^2_{b_t}\widehat Q\right\|_{\mathrm{op}}\,\|b-b'\|_2
=
\frac{\lambda_{\max}(\widehat M_{\pa(t)})}{(\sigma_t^2)^{(r)}}\,\|b-b'\|_2,
\]
so $\widehat Q(\cdot\mid \vartheta^{(r)})$ is $L^{(r)}$-smooth with
\[
L^{(r)}=\frac{\lambda_{\max}(\widehat M_{\pa(t)})}{(\sigma_t^2)^{(r)}}.
\]

Finally, for a concave function with $L^{(r)}$-Lipschitz gradient, the standard smoothness inequality implies that for the gradient-ascent update
$b_t^{(r+1)}=b_t^{(r)}+\eta_r \nabla_{b_t}\widehat Q(b_t^{(r)}\mid \vartheta^{(r)})$ with $0<\eta_r\le 2/L^{(r)}$,
\[
\widehat Q(b_t^{(r+1)}\mid \vartheta^{(r)})
\ge
\widehat Q(b_t^{(r)}\mid \vartheta^{(r)})
+\left(\eta_r-\frac{L^{(r)}\eta_r^2}{2}\right)\left\|\nabla_{b_t}\widehat Q(b_t^{(r)}\mid \vartheta^{(r)})\right\|_2^2
\ge
\widehat Q(b_t^{(r)}\mid \vartheta^{(r)}),
\]
since $\eta_r-\frac{L^{(r)}\eta_r^2}{2}\ge 0$ when $\eta_r\le 2/L^{(r)}$.
Thus the one-step update is monotone ascent on the surrogate and hence defines a valid GEM step \cite{Dempster1977,Wu1983}.
\end{proof}

\begin{proof}[Proof of Theorem~\ref{thm:population_contraction_param}]
Throughout, work on the ball $\mathbb{B}(\theta_t^*;r)$ where the assumptions hold.

\paragraph{(1) Exact EM operator.}
Fix any $\theta_t\in \mathbb{B}(\theta_t^*;r)$ and define
\[
F_t(\theta_t)\in\arg\max_{\theta\in\mathbb{B}(\theta_t^*;r)} \bar Q_t(\theta\mid\theta_t).
\]
By Assumption~\ref{assump:local_sc_smooth}, $\theta\mapsto \bar Q_t(\theta\mid\theta_t)$ is $\lambda$-strongly concave on the ball,
so the maximizer is unique and satisfies the first-order optimality condition
\begin{equation}\label{eq:FOC_Ft_app_rev}
\nabla_1 \bar Q_t(F_t(\theta_t)\mid \theta_t)=0.
\end{equation}
Also, $\theta_t^*$ is a population stationary point, so
\begin{equation}\label{eq:FOC_star_app_rev}
\nabla_1 \bar Q_t(\theta_t^*\mid \theta_t^*)=0.
\end{equation}

Consider
\[
0-\nabla_1 \bar Q_t(\theta_t^*\mid \theta_t)
=
\nabla_1 \bar Q_t(F_t(\theta_t)\mid \theta_t)-\nabla_1 \bar Q_t(\theta_t^*\mid \theta_t),
\]
using \eqref{eq:FOC_Ft_app_rev}. Taking inner product with $F_t(\theta_t)-\theta_t^*$ and applying $\lambda$-strong concavity
in the first argument yields
\[
\Big\langle \nabla_1 \bar Q_t(F_t(\theta_t)\mid \theta_t)-\nabla_1 \bar Q_t(\theta_t^*\mid \theta_t),\;
F_t(\theta_t)-\theta_t^* \Big\rangle
\le -\lambda\|F_t(\theta_t)-\theta_t^*\|_2^2.
\]
By Cauchy--Schwarz,
\[
\Big\langle -\nabla_1 \bar Q_t(\theta_t^*\mid \theta_t),\; F_t(\theta_t)-\theta_t^* \Big\rangle
\le \|\nabla_1 \bar Q_t(\theta_t^*\mid \theta_t)\|_2\,\|F_t(\theta_t)-\theta_t^*\|_2.
\]
Combining gives
\[
\lambda\|F_t(\theta_t)-\theta_t^*\|_2 \le \|\nabla_1 \bar Q_t(\theta_t^*\mid \theta_t)\|_2.
\]
Add and subtract $\nabla_1 \bar Q_t(\theta_t^*\mid \theta_t^*)=0$ and apply Assumption~\ref{assump:grad_stability}:
\[
\|\nabla_1 \bar Q_t(\theta_t^*\mid \theta_t)\|_2
=
\|\nabla_1 \bar Q_t(\theta_t^*\mid \theta_t)-\nabla_1 \bar Q_t(\theta_t^*\mid \theta_t^*)\|_2
\le \gamma \|\theta_t-\theta_t^*\|_2.
\]
Therefore,
\[
\|F_t(\theta_t)-\theta_t^*\|_2 \le (\gamma/\lambda)\,\|\theta_t-\theta_t^*\|_2,
\]
which proves contraction. The fixed-point and geometric convergence follow by Banach’s theorem.

\paragraph{(2) Block first-order / gradient-EM coefficient update.}
Let $b_t^+=b_t+\eta\nabla_{b_t}\bar Q_t(\theta_t\mid\theta_t)$ with $0<\eta\le 1/\mu$.
Add and subtract $\nabla_{b_t}\bar Q_t(\theta_t\mid\theta_t^*)$:
\begin{align}
\|b_t^+-b_t^*\|_2
&\le
\underbrace{\Big\|b_t-b_t^* + \eta\big(\nabla_{b_t}\bar Q_t(\theta_t\mid\theta_t^*)
-\nabla_{b_t}\bar Q_t(\theta_t^*\mid\theta_t^*)\big)\Big\|_2}_{(\star)}
+
\eta\underbrace{\Big\|\nabla_{b_t}\bar Q_t(\theta_t\mid\theta_t)
-\nabla_{b_t}\bar Q_t(\theta_t\mid\theta_t^*)\Big\|_2}_{(\dagger)}.
\label{eq:bt_split_app_rev}
\end{align}

\emph{Control of $(\star)$.}
Fix $\alpha_t$ and define $g(b):=\bar Q_t((b,\alpha_t)\mid\theta_t^*)$.
By Assumption~\ref{assump:local_sc_smooth}, $g$ is $\lambda$-strongly concave and $\mu$-smooth in $b$ on the ball.
Hence for $0<\eta\le 1/\mu$, the gradient-ascent map $b\mapsto b+\eta\nabla g(b)$ is a contraction with factor $(1-\eta\lambda)$, so
\[
(\star)\le (1-\eta\lambda)\,\|b_t-b_t^*\|_2.
\]

\emph{Control of $(\dagger)$.}
Apply Assumption~\ref{assump:grad_stability} with $\theta_t'=\theta_t$:
\begin{equation}\label{eq:dagger_bound_app_rev}
(\dagger)
=
\Big\|\nabla_{b_t}\bar Q_t(\theta_t\mid\theta_t)
-\nabla_{b_t}\bar Q_t(\theta_t\mid\theta_t^*)\Big\|_2
\le \gamma\,\|\theta_t-\theta_t^*\|_2.
\end{equation}

Combining the last three displays and using $\|b_t-b_t^*\|_2\le \|\theta_t-\theta_t^*\|_2$ gives
\[
\|b_t^+-b_t^*\|_2
\le
(1-\eta\lambda)\,\|\theta_t-\theta_t^*\|_2+\eta\gamma\,\|\theta_t-\theta_t^*\|_2
=
\bigl(1-\eta(\lambda-\gamma)\bigr)\,\|\theta_t-\theta_t^*\|_2,
\]
as claimed.

Finally, if the $\alpha_t$-update is itself contractive with factor $\rho_\alpha<1$ on the ball, then under the product Euclidean norm,
\[
\|G_t(\theta_t)-\theta_t^*\|_2
=
\big\|(b_t^+,\alpha_t^+)-(b_t^*,\alpha_t^*)\big\|_2
\le
\max\{1-\eta(\lambda-\gamma),\rho_\alpha\}\,\|\theta_t-\theta_t^*\|_2,
\]
so $G_t$ is contractive.
\end{proof}

\begin{proof}[Proof of Lemma \ref{lem:lambda_mu}]
Fix any $\theta_t' \in \mathbb{B}(\theta_t^*;r)$ and write $\theta_t=(b_t,\alpha_t)$ with
$\sigma_t^2 := e^{\alpha_t}\in[\Delta_{\min},\Delta_{\max}]$ by assumption.  For the local linear-Gaussian mechanism
$T\mid X_{\pa(t)} \sim \mathcal{N}(b_t^\top X_{\pa(t)},\,\sigma_t^2)$, the (population) EM surrogate restricted to block $t$
can be written (up to additive terms independent of $(b_t,\alpha_t)$) as
\begin{equation}\label{eq:Qbar_local_form}
\bar Q_t(b_t,\alpha_t \mid \theta_t')
\;=\;
-\frac12\,\mathbb{E}\!\left[
\alpha_t \;+\; e^{-\alpha_t}\,\widetilde r_t(b_t;\theta_t')^2
\right] \;+\; \text{const}(\theta_t'),
\end{equation}
where $\widetilde r_t(b_t;\theta_t')$ is the E-step residual (completed-data moment) and denotes the (population) residual random variable appearing in the surrogate
(e.g., the E-step conditional second moment of $T-b_t^\top X_{\pa(t)}$ given the observed variables, under $\theta_t'$).
Crucially, for fixed $\theta_t'$, $\theta_t\mapsto \bar Q_t(\theta_t\mid\theta_t')$ is twice differentiable and its curvature
in $(b_t,\alpha_t)$ is determined by the second derivatives of the right-hand side of \eqref{eq:Qbar_local_form}.

\paragraph{Curvature in the $b_t$-coordinate.}
Differentiating \eqref{eq:Qbar_local_form} with respect to $b_t$ gives
\[
\nabla_{b_t}\bar Q_t(b_t,\alpha_t\mid\theta_t')
=
e^{-\alpha_t}\,\mathbb{E}\!\left[X_{\pa(t)}\,\widetilde r_t(b_t;\theta_t')\right],
\]
and the Hessian in $b_t$ is the constant (in $b_t$) negative semidefinite matrix
\[
\nabla^2_{b_tb_t}\bar Q_t(b_t,\alpha_t\mid\theta_t')
=
-\,e^{-\alpha_t}\,\mathbb{E}\!\left[X_{\pa(t)}X_{\pa(t)}^\top\right].
\]
By the moment bounds $mI\preceq \mathbb{E}[X_{\pa(t)}X_{\pa(t)}^\top]\preceq MI$ and the variance bounds
$e^{-\alpha_t}\in[1/\Delta_{\max},\,1/\Delta_{\min}]$, we obtain the uniform spectral bounds
\[
-\frac{M}{\Delta_{\min}}I
\;\preceq\;
\nabla^2_{b_tb_t}\bar Q_t(b_t,\alpha_t\mid\theta_t')
\;\preceq\;
-\frac{m}{\Delta_{\max}}I,
\]
which implies $b_t\mapsto \bar Q_t(b_t,\alpha_t\mid\theta_t')$ is $\lambda_b$-strongly concave and $\mu_b$-smooth with
\[
\lambda_b \;\ge\; \frac{m}{\Delta_{\max}},
\qquad
\mu_b \;\le\; \frac{M}{\Delta_{\min}}.
\]

\paragraph{Curvature in the $\alpha_t$-coordinate.}
For fixed $b_t$, differentiate \eqref{eq:Qbar_local_form} with respect to $\alpha_t$:
\[
\partial_{\alpha_t}\bar Q_t(b_t,\alpha_t\mid\theta_t')
=
-\frac12 \;+\; \frac12\,e^{-\alpha_t}\,\mathbb{E}\!\left[\widetilde r_t(b_t;\theta_t')^2\right],
\]
and
\[
\partial_{\alpha_t}^2\bar Q_t(b_t,\alpha_t\mid\theta_t')
=
-\frac12\,e^{-\alpha_t}\,\mathbb{E}\!\left[\widetilde r_t(b_t;\theta_t')^2\right]
\;\le\;0.
\]
By the assumed uniform residual-moment bounds
$0<v_{\min}\le \mathbb{E}[\widetilde r_t(b_t;\theta_t')^2]\le v_{\max}<\infty$ on the ball (for all $\theta_t$) and again
$e^{-\alpha_t}\in[1/\Delta_{\max},\,1/\Delta_{\min}]$, we obtain
\[
-\frac12\,\frac{v_{\max}}{\Delta_{\min}}
\;\le\;
\partial_{\alpha_t}^2\bar Q_t(b_t,\alpha_t\mid\theta_t')
\;\le\;
-\frac12\,\frac{v_{\min}}{\Delta_{\max}}.
\]
Hence $\alpha_t\mapsto \bar Q_t(b_t,\alpha_t\mid\theta_t')$ is $\lambda_\alpha$-strongly concave and $\mu_\alpha$-smooth with
\[
\lambda_\alpha \;\ge\; \frac12\,\frac{v_{\min}}{\Delta_{\max}},
\qquad
\mu_\alpha \;\le\; \frac12\,\frac{v_{\max}}{\Delta_{\min}}.
\]

\paragraph{From blockwise to full-block curvature (Schur complement).}
Let $H(\theta_t;\theta_t'):=\nabla^2_{\theta_t\theta_t}\bar Q_t(\theta_t\mid\theta_t')$ and write it in block form
\[
H(\theta_t;\theta_t')
=
\begin{pmatrix}
H_{bb} & H_{b\alpha}\\
H_{\alpha b} & H_{\alpha\alpha}
\end{pmatrix},
\qquad
H_{bb}=\nabla^2_{b_tb_t}\bar Q_t,\;\;
H_{\alpha\alpha}=\partial_{\alpha_t}^2\bar Q_t,\;\;
H_{b\alpha}=\nabla^2_{b_t\alpha_t}\bar Q_t.
\]
From the bounds above, uniformly on the ball,
\[
H_{bb}\preceq -\lambda_b I,\qquad H_{\alpha\alpha}\le -\lambda_\alpha,
\qquad
\|H_{b\alpha}\|_2 \le \rho.
\]
If $\rho^2<\lambda_b\lambda_\alpha$, then by a standard Schur-complement argument the whole Hessian is uniformly negative definite
on $\mathbb{B}(\theta_t^*;r)$; for example one may take the strong concavity constant
\[
\lambda \;:=\; \frac12\Big(\lambda_b+\lambda_\alpha-\sqrt{(\lambda_b-\lambda_\alpha)^2+4\rho^2}\Big) \;>\; 0,
\]
so that $H(\theta_t;\theta_t')\preceq -\lambda I$ on the ball. Similarly, using the upper smoothness bounds
$\|H_{bb}\|_{\mathrm{op}}\le \mu_b$, $|H_{\alpha\alpha}|\le \mu_\alpha$, and $\|H_{b\alpha}\|_2\le \rho$, one can take
\[
\mu \;:=\; \frac12\Big(\mu_b+\mu_\alpha+\sqrt{(\mu_b-\mu_\alpha)^2+4\rho^2}\Big)
\]
to obtain $\|H(\theta_t;\theta_t')\|_{\mathrm{op}}\le \mu$ uniformly on the ball. Therefore,
Assumption~\ref{assump:local_sc_smooth} holds for the full block $\theta_t=(b_t,\alpha_t)$ with constants depending on
$(\lambda_b,\mu_b,\lambda_\alpha,\mu_\alpha,\rho)$.
\end{proof}

\begin{proof}[Proof of Proposition \ref{prop:gamma_bound}]
Recall Assumption~\ref{assump:grad_stability} (restricted to the $b_t$-coordinate) requires that for all
$\theta,\theta'\in\mathbb{B}(\theta_t^*;r)$,
\[
\bigl\|\nabla_{b_t}\bar Q_t(\theta_t\mid \theta')-\nabla_{b_t}\bar Q_t(\theta_t\mid \theta)\bigr\|_2
\;\le\; \gamma\,\|\theta'-\theta\|_2,
\]
uniformly for $\theta_t\in\mathbb{B}(\theta_t^*;r)$.

Fix $\theta,\theta'\in\mathbb{B}(\theta_t^*;r)$ and any $\theta_t=(b_t,\alpha_t)\in\mathbb{B}(\theta_t^*;r)$.
For the local linear-Gaussian mechanism, the population surrogate gradient in $b_t$ has the form
\begin{equation}\label{eq:gradQ_b_form}
\nabla_{b_t}\bar Q_t(\theta_t\mid \vartheta)
=
e^{-\alpha_t}\,\mathbb{E}\!\left[X_{\pa(t)}\Big(\mu_{\vartheta}(X_{-t})-b_t^\top X_{\pa(t)}\Big)\right],
\end{equation}
where $\mu_{\vartheta}(x_{-t})=\mathbb{E}_{\vartheta}[T\mid X_{-t}=x_{-t}]$ denotes the E-step conditional mean under parameter
$\vartheta$ (and the expectation is over the population distribution of $X$).

Subtracting \eqref{eq:gradQ_b_form} at $\vartheta=\theta'$ and $\vartheta=\theta$ cancels the $b_t^\top X_{\pa(t)}$ term, yielding
\[
\nabla_{b_t}\bar Q_t(\theta_t\mid \theta')-\nabla_{b_t}\bar Q_t(\theta_t\mid \theta)
=
e^{-\alpha_t}\,\mathbb{E}\!\left[X_{\pa(t)}\big(\mu_{\theta'}(X_{-t})-\mu_{\theta}(X_{-t})\big)\right].
\]
Taking norms and applying Jensen / triangle inequality gives
\[
\bigl\|\nabla_{b_t}\bar Q_t(\theta_t\mid \theta')-\nabla_{b_t}\bar Q_t(\theta_t\mid \theta)\bigr\|_2
\;\le\;
e^{-\alpha_t}\,\mathbb{E}\!\left[\|X_{\pa(t)}\|_2\,\big|\mu_{\theta'}(X_{-t})-\mu_{\theta}(X_{-t})\big|\right].
\]
By the envelope Lipschitz condition in the proposition,
\[
\big|\mu_{\theta'}(x_{-t})-\mu_{\theta}(x_{-t})\big|
\;\le\; L_\mu(x_{-t})\,\|\theta'-\theta\|_2
\quad\forall x_{-t},
\]
so
\[
\bigl\|\nabla_{b_t}\bar Q_t(\theta_t\mid \theta')-\nabla_{b_t}\bar Q_t(\theta_t\mid \theta)\bigr\|_2
\;\le\;
e^{-\alpha_t}\,\mathbb{E}\!\left[\|X_{\pa(t)}\|_2\,L_\mu(X_{-t})\right]\;\|\theta'-\theta\|_2.
\]
On $\mathbb{B}(\theta_t^*;r)$ we have $\alpha_t\ge \alpha_{\min}:=\log\Delta_{\min}$, hence
$e^{-\alpha_t}\le e^{-\alpha_{\min}} = 1/\Delta_{\min}$. Therefore, uniformly over $\theta_t$ in the ball,
\[
\bigl\|\nabla_{b_t}\bar Q_t(\theta_t\mid \theta')-\nabla_{b_t}\bar Q_t(\theta_t\mid \theta)\bigr\|_2
\;\le\;
e^{-\alpha_{\min}}\;\mathbb{E}\!\left[\|X_{\pa(t)}\|_2\,L_\mu(X_{-t})\right]\;\|\theta'-\theta\|_2
\;\le\;
\frac{1}{\Delta_{\min}}\;\mathbb{E}\!\left[\|X_{\pa(t)}\|_2\,L_\mu(X_{-t})\right]\;\|\theta'-\theta\|_2.
\]
Thus Assumption~\ref{assump:grad_stability} holds with
\[
\gamma \;\le\; e^{-\alpha_{\min}}\;\mathbb{E}\!\left[\|X_{\pa(t)}\|_2\,L_\mu(X_{-t})\right]
\;\le\;
\frac{1}{\Delta_{\min}}\;\mathbb{E}\!\left[\|X_{\pa(t)}\|_2\,L_\mu(X_{-t})\right],
\]
as claimed.
\end{proof}

\begin{proof}[Proof of Lemma \ref{lem:lipschitz_mu}]
Write
\[
A(\theta_t)\;:=\;K_{tt}(\theta_t)^{-1}K_{t,-t}(\theta_t)\in\mathbb{R}^{1\times (p-1)}.
\]
Then the conditional mean can be written as
\[
\mu_{\theta_t}(x_{-t})
=
m_t(\theta_t)-A(\theta_t)\bigl(x_{-t}-m_{-t}(\theta_t)\bigr).
\]
Fix $\theta_t,\theta_t'\in\mathbb{B}(\theta_t^\ast;r)$ and abbreviate $x:=x_{-t}$.
Add and subtract $m_{-t}(\theta_t^\ast)$ to isolate the $x$-dependence:
\begin{align*}
\mu_{\theta_t}(x)
&=
m_t(\theta_t)-A(\theta_t)\bigl(x-m_{-t}(\theta_t^\ast)\bigr)
\;+\;
A(\theta_t)\bigl(m_{-t}(\theta_t)-m_{-t}(\theta_t^\ast)\bigr).
\end{align*}
Hence
\begin{align*}
\mu_{\theta_t'}(x)-\mu_{\theta_t}(x)
&=
\underbrace{\bigl(m_t(\theta_t')-m_t(\theta_t)\bigr)}_{(I)}
\;-\;
\underbrace{\bigl(A(\theta_t')-A(\theta_t)\bigr)\bigl(x-m_{-t}(\theta_t^\ast)\bigr)}_{(II)}\\
&\quad
+\underbrace{\Big(A(\theta_t')\bigl(m_{-t}(\theta_t')-m_{-t}(\theta_t^\ast)\bigr)
      -A(\theta_t)\bigl(m_{-t}(\theta_t)-m_{-t}(\theta_t^\ast)\bigr)\Big)}_{(III)}.
\end{align*}
We bound each term.

\emph{Term (I).} By the mean value theorem and the bound
$\sup_{\theta_t\in\mathbb{B}(\theta_t^\ast;r)}\|\nabla_{\theta_t}m(\theta_t)\|_{\mathrm{op}}\le C_m$,
\[
|m_t(\theta_t')-m_t(\theta_t)|\;\le\;\|m(\theta_t')-m(\theta_t)\|_2
\;\le\; C_m\|\theta_t'-\theta_t\|_2.
\]

\emph{Term (II).} Using the mean value theorem and the bound
$\sup_{\theta_t\in\mathbb{B}(\theta_t^\ast;r)}\|\nabla_{\theta_t}A(\theta_t)\|_{\mathrm{op}}\le C_K$,
\[
\|A(\theta_t')-A(\theta_t)\|_2 \;\le\; C_K\|\theta_t'-\theta_t\|_2,
\]
hence
\[
|(II)|
\;\le\;
\|A(\theta_t')-A(\theta_t)\|_2\,\|x-m_{-t}(\theta_t^\ast)\|_2
\;\le\;
C_K\|x-m_{-t}(\theta_t^\ast)\|_2\,\|\theta_t'-\theta_t\|_2.
\]

\emph{Term (III).} First note that $A(\cdot)$ is continuous on the compact set
$\mathbb{B}(\theta_t^\ast;r)$ and $K_{tt}(\theta_t)\ge c_K>0$ on the ball, so
\[
C_A:=\sup_{\theta_t\in\mathbb{B}(\theta_t^\ast;r)}\|A(\theta_t)\|_2 < \infty.
\]
Now add and subtract $A(\theta_t')\bigl(m_{-t}(\theta_t)-m_{-t}(\theta_t^\ast)\bigr)$ to get
\begin{align*}
(III)
&=
A(\theta_t')\bigl(m_{-t}(\theta_t')-m_{-t}(\theta_t)\bigr)
+\bigl(A(\theta_t')-A(\theta_t)\bigr)\bigl(m_{-t}(\theta_t)-m_{-t}(\theta_t^\ast)\bigr).
\end{align*}
Therefore,
\begin{align*}
|(III)|
&\le
\|A(\theta_t')\|_2\,\|m_{-t}(\theta_t')-m_{-t}(\theta_t)\|_2
+\|A(\theta_t')-A(\theta_t)\|_2\,\|m_{-t}(\theta_t)-m_{-t}(\theta_t^\ast)\|_2\\
&\le
C_A\cdot C_m\|\theta_t'-\theta_t\|_2
+ C_K\|\theta_t'-\theta_t\|_2\cdot \|m_{-t}(\theta_t)-m_{-t}(\theta_t^\ast)\|_2.
\end{align*}
Finally, $\|m_{-t}(\theta_t)-m_{-t}(\theta_t^\ast)\|_2\le C_m\|\theta_t-\theta_t^\ast\|_2\le C_m r$ on the ball, so
\[
|(III)| \;\le\; \bigl(C_A C_m + C_K C_m r\bigr)\,\|\theta_t'-\theta_t\|_2.
\]

Putting the three bounds together yields, for all $x_{-t}$,
\[
|\mu_{\theta_t'}(x_{-t})-\mu_{\theta_t}(x_{-t})|
\;\le\;
\Bigl(C_m + C_A C_m + C_K C_m r\Bigr)\|\theta_t'-\theta_t\|_2
\;+\;
C_K\|x_{-t}-m_{-t}(\theta_t^\ast)\|_2\,\|\theta_t'-\theta_t\|_2.
\]
Thus the desired Lipschitz-envelope bound holds with
\[
L_\mu(x_{-t})
:=
C_0 + C_K\|x_{-t}-m_{-t}(\theta_t^\ast)\|_2,
\qquad
C_0:= C_m + C_A C_m + C_K C_m r,
\]
and (equivalently) you may keep the form $L_\mu(x_{-t})=C_m+C_K\|x_{-t}-m_{-t}(\theta_t^\ast)\|_2$
by redefining $C_m$ to absorb $C_0$.

Finally, if $\mathbb{E}\|X_{-t}\|_2^2<\infty$, then by Cauchy--Schwarz,
\[
\mathbb{E}\|X_{-t}-m_{-t}(\theta_t^\ast)\|_2 \;\le\; \big(\mathbb{E}\|X_{-t}-m_{-t}(\theta_t^\ast)\|_2^2\big)^{1/2}<\infty,
\]
so $\mathbb{E}[L_\mu(X_{-t})]<\infty$. This verifies the envelope condition required by
Proposition~\ref{prop:gamma_bound}.
\end{proof}

\bibliographystyle{IEEEtran}  
\bibliography{mybibfile}
\end{document}